\documentclass{article}

\usepackage{microtype}
\usepackage{graphicx}
\usepackage{subfigure}
\usepackage{booktabs} 

\usepackage{hyperref}

\usepackage[accepted]{icml2023}

\usepackage{amsmath}
\usepackage{amssymb}
\usepackage{mathtools}
\usepackage{amsthm}
\usepackage{multirow}

\usepackage[capitalize,noabbrev]{cleveref}

\theoremstyle{plain}
\newtheorem{theorem}{Theorem}[section]
\newtheorem{proposition}[theorem]{Proposition}
\newtheorem{lemma}[theorem]{Lemma}
\newtheorem{corollary}[theorem]{Corollary}
\theoremstyle{definition}
\newtheorem{definition}[theorem]{Definition}

\theoremstyle{remark}

\usepackage[textsize=tiny]{todonotes}

\usepackage{nicefrac}

\newcommand{\bfx}{\mathbf{x}}
\newcommand{\bfz}{\mathbf{z}}

\newcommand{\bfr}{\mathbf{r}}

\newcommand{\bfu}{\mathbf{u}}

\newcommand{\bfa}{\mathbf{a}}

\newcommand{\bft}{\mathbf{t}}
\newcommand{\bff}{\mathbf{f}}
\newcommand{\bfy}{\mathbf{y}}
\newcommand{\bfK}{\mathbf{K}}
\newcommand{\bfI}{\mathbf{I}}
\newcommand{\bfC}{\mathbf{C}}

\newcommand{\bfA}{\mathbf{A}}
\newcommand{\bfD}{\mathbf{D}}

\newcommand{\bfS}{\mathbf{S}}

\newcommand{\bfN}{\mathbf{N}}

\newcommand{\bfL}{\mathbf{L}}

\newcommand{\first}[1]{\textbf{\textcolor{red}{#1}}}
\newcommand{\second}[1]{\textbf{\textcolor{violet}{#1}}}
\newcommand{\third}[1]{\textbf{\textcolor{black}{#1}}}

\icmltitlerunning{Graph Positional Encoding via Random Feature Propagation}

\begin{document}

\twocolumn[
\icmltitle{Graph Positional Encoding via Random Feature Propagation}




\begin{icmlauthorlist}
\icmlauthor{Moshe Eliasof}{bgu}
\icmlauthor{Fabrizio Frasca}{ff}
\icmlauthor{Beatrice Bevilacqua}{bb}
\icmlauthor{Eran Treister}{bgu}
\icmlauthor{Gal Chechik}{biu,comp}
\icmlauthor{Haggai Maron}{comp}

\end{icmlauthorlist}

\icmlaffiliation{bgu}{Department of Computer Science, Ben-Gurion University of the Negev, Israel.}
\icmlaffiliation{biu}{Department of Computer Science, Bar-Ilan University, Israel.}

\icmlaffiliation{bb}{Department of Computer Science, Purdue University, USA.}
\icmlaffiliation{ff}{Department of Computing, Imperial College London, UK.}

\icmlaffiliation{comp}{NVIDIA Research}

\icmlcorrespondingauthor{Moshe Eliasof}{eliasof@post.bgu.ac.il}

\icmlkeywords{Machine Learning, ICML}

\vskip 0.3in
]



\printAffiliationsAndNotice{} 

\begin{abstract}
Two main families of node feature augmentation schemes have been explored for enhancing GNNs: random features and spectral positional encoding. Surprisingly, however, there is still no clear understanding of the relation between these two augmentation schemes. Here we propose a novel family of positional encoding schemes which draws a link between the above two approaches and improves over both. The new approach, named Random Feature Propagation (RFP), is inspired by the power iteration method and its generalizations. It concatenates several intermediate steps of an iterative algorithm for computing the dominant eigenvectors of a propagation matrix, starting from random node features. Notably, these propagation steps are based on graph-dependent propagation operators that can be either predefined or learned. 
We explore the theoretical and empirical benefits of RFP. First, we provide theoretical justifications for using random features, for incorporating early propagation steps, and for using multiple random initializations.  Then, we empirically demonstrate that RFP significantly outperforms both spectral PE and random features in multiple node classification and graph classification benchmarks. 
\end{abstract}

\section{Introduction}
\label{sec:intro}
GNN architectures such as Message-Passing Neural Networks (MPNNs) became very popular for learning with graph-structured data, but they suffer from various limitations. Primarily, they were shown to have limited expressiveness, which hurts their performance in practice \cite{morris2019weisfeiler, xu2019how,morris2021weisfeiler,zhang2021labeling}.
Many approaches have been proposed to alleviate these limitations, including two prominent directions that are based on augmenting the input node features of a given graph, namely, random and spectral features. The first approach builds on 
random node features (RNF) \cite{abboud2020surprising,sato2021random}, and was shown to improve the expressiveness of MPNNs by increasing their capability to disambiguate nodes. Unfortunately, even though RNF theoretically add expressiveness, they do not consistently improve performance on real-world datasets \cite{bevilacqua2021equivariant}. A second approach to alleviate the  expressivity limitation suggests enhancing input node features with spectral PE  often derived from a partial eigendecomposition of the graph-Laplacian matrix \cite{dwivedi2020generalization,kreuzer2021rethinking,wang2022equivariant,lim2022sign,rampasek2022GPS}. These methods have demonstrated an improvement on graph classification benchmarks compared to using only the raw input node features. 

\textbf{Our approach.} This work is based on the observation that iterative matrix eigenproblem solvers provide a natural link between random node features and existing spectral PE methods. Prominent examples of these iterative algorithms include the well-known power method and its generalizations, such as subspace iteration methods \cite{saad2011numerical}.
Specifically, these eigensolvers take random vectors and transform them into eigenvectors by performing a series of propagation and normalization steps. Importantly, when the input matrix represents either the graph adjacency or the Laplacian matrix, the output of these algorithms is the spectral PE mentioned above.

\begin{figure}[t]  \includegraphics[width=0.75\linewidth]{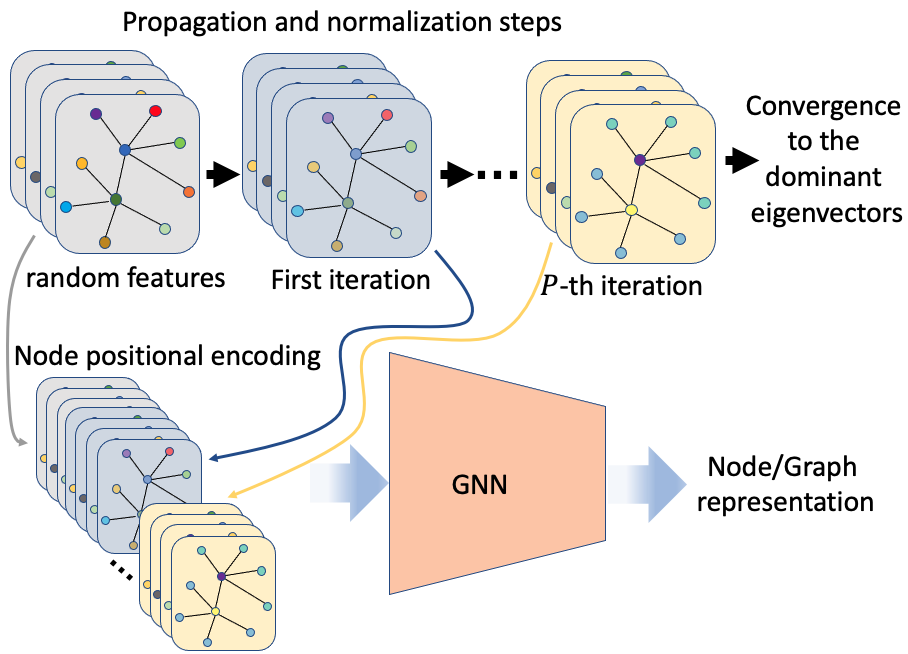}
    \centering
    \caption{An overview of our Random Feature Propagation (RFP). We start from $k$-dimensional random node features and perform a series of propagation and normalization steps according to a predefined or learned propagation operator. We then use the whole trajectory as node PE. As a result of this process, the network has access to features ranging from pure random features to features that approximate the dominant $k$ eigenvectors of the propagation operator. We also suggest methods for learning the propagation operator and for processing multiple trajectories.
    }
    \label{fig:fig1}
\end{figure}
In this paper, we propose a family of PE schemes for graph learning, which we call \textit{Random Feature Propagation} (RFP), designed to improve over RNF and eigendecomposition-based PE schemes while retaining the advantages of both approaches. 
Given a graph with $n$ nodes, an $n\times n$ propagation matrix (e.g., the adjacency matrix),  and a normalization function (e.g., $\ell_2$ normalization or orthonormalization), the RFP PE is defined as a concatenation of multiple successive propagation-and-normalization steps, starting from random node features. We refer to the output of this process as the \emph{RFP trajectory}. As we show in \cref{sec:overview}, this approach can be interpreted as an interpolation between the RNF and the spectral PE strategies outlined above. Therefore, RFP provides GNNs with features ranging from pure RNF to approximations of the dominant eigenvectors of the propagation operator. Figure \ref{fig:fig1} illustrate this concept. We show empirically and theoretically that RFP combines the benefits of both approaches. In particular, we demonstrate that feeding the GNN with the RFP trajectory is more effective than feeding it with RNF and/or the eigenvectors only.

A key component of our framework is the propagation operator. We propose two types of propagation operators: the first kind is a pre-defined operator, such as the adjacency or Laplacian matrices. The second kind of operators are \emph{learnable}, data-dependent, propagation operators that are learned jointly with the GNN in an end-to-end manner according to the downstream task. Learning the operator is performed by incorporating an attention mechanism that outputs scores for pairs of nodes. In contrast to standard propagation operators, the learned one can model additional node interactions, including long-range pairwise dependencies.

We demonstrate that RFP can be used successfully with standard GNNs such as GIN \cite{xu2019how} and GCN \cite{kipf2016semi}. Furthermore, we propose using \emph{multiple} RFP trajectories as PE and processing them with a permutation equivariant architecture called DSS-GNN, which was recently suggested \cite{bevilacqua2021equivariant} for other purposes. The combination of multiple RFP trajectories with DSS-GNN further improves the performance.

We present several theoretical results to support the main choices made in our framework. As a first result, we show that early stages in the RFP trajectory are useful: we prove that MPNNs can use these early-stage features in order to extract structural information, such as the number of triangles in the graph. Our second result provides theoretical support for starting the propagation steps from random features rather than the original node features, as suggested in recent works \cite{rossi2020sign,huang2022from}. In particular, we show that unlike in these previous works, our trajectories converge almost surely to any chosen number of dominant eigenvectors of the propagation operator. As a result, we are able to apply our method to graphs whose input feature matrices have low rank or dimensionality, such as featureless graphs. Third, our framework trivially inherits the strong approximation properties of networks that use random node features as input \cite{abboud2020surprising}.%

Our comprehensive experimental study explores the effects of the propagation and normalization operators, and the number of steps in the RFP trajectory, on both node-level and graph-level tasks. The results indicate that RFP  improves the performance of MPNNs, with or without spectral and random encoding as well as other recent methods.

\textbf{Contributions.} The contributions of this work are as follows: (1) We develop RFP --- a novel PE scheme that combines the benefits of RNF and spectral PE, based on iterative propagation and normalization of random features; (2) We describe a method for learning the propagation operator rather than using a predefined one; 
(3) We present a theoretical analysis of our PE, giving insight into the design choices we have made 
; and (4)  
We evaluate RFP in a series of node and graph classification tasks demonstrating its benefit.

\section{Related Work}

\textbf{Random node features.}
Random initializations are widely used in stochastic approximation algorithms such as matrix trace estimation \cite{hutchinson1989stochastic} and subgraph counting \cite{avron2010counting}. 
In the context of GNNs, random input features were used to improve the expressiveness of MPNNs \cite{dasoulas2019coloring,puny2020global, abboud2020surprising, sato2021random}.
In particular, \citet{puny2020global, abboud2020surprising}, proved a universal approximation theorem (with high probability) for GNNs that use random node features. While theoretically powerful, such a strategy did not yield consistently improved performance on real-world datasets (see \citet{bevilacqua2021equivariant} and Section \ref{sec:experiments}).

\textbf{Spectral positional encoding.}
Embedding nodes using the eigendecomposition of graph operators is a classical technique in data analysis, see e.g.,  \citet{belkin2003laplacian,coifman2006diffusion}.
As these embeddings contain valuable information, recent work has proposed using them as PE to overcome the limitations of MPNNs~\cite{dwivedi2020benchmarkgnns} and as inputs to transformers when applied to graph data~\cite{dwivedi2020generalization}. Popular instances of these methods suggest using the graph Laplacian eigenvectors as initial node features~\citep{dwivedi2020benchmarkgnns,dwivedi2020generalization}, but several improvements have been recently proposed~\citep{kreuzer2021rethinking, wang2022equivariant, lim2022sign,rampasek2022GPS,maskey2022generalized}. These methods use eigenvectors as PE, while we demonstrate that using the trajectory towards computing these eigenvectors, including the initial RNF, yields better performance, and has favorable theoretical properties.

\textbf{Successive application of propagation operators.} The concept of utilizing node features that underwent several applications of graph operators has been proposed and studied in previous papers~\citep{zhou2003learning, wang2020unifying}. 
Differently from our procedure, the output of the propagation steps for these methods is solely determined by the input graph connectivity and partially known node labels, and involves no consideration of input or random node features.
The combination of original input node features and propagation steps has been explored in \citet{klicpera2018predict, wu2019simplifying,rossi2020sign,huang2021combining}. However, these works have different motivations, propose substantially different architectures, and do not make use of random node features. Another recent paper \cite{dwivedi2022graph} suggests using learnable layers in order to enhance positional encoding based on the powers of random walk matrices.

 Closely related to our work is a very recent work by \citet{huang2022from} that adds a matrix inversion-based normalization step after each propagation step, resulting in a propagation-normalization procedure similar to ours. The authors show that when the node feature matrix has full rank, their iterative process converges to the dominant $k$ eigenvectors of the used operator, where $k$ is the number of initial node features, and present an  accuracy improvement over previous approaches. Their approach, however, relies on the \emph{original} raw input node features. In contrast, our approach advocates the use of RNF, different normalization schemes, and several algorithmic enhancements like learning the propagation operator. In Section \ref{sec:theory} we prove why starting from random features is a better choice than the original node features in some cases. In Section \ref{sec:experiments} we show that our approach achieves better results in practice.

\section{Method}
The following section describes the proposed method. We begin with a brief overview of the approach before discussing its main components in more detail.

\textbf{Notation.}\label{sec:notations}
An undirected graph is defined by the tuple $\cal G=({\cal V},{\cal E})$  where $\cal V$ is a set of $n$ vertices and $\cal E$ is a set of $m$ edges. Let us denote by ${\bf f}_i\in\mathbb{R}^c$ the feature vector of the $i$-th node of $\cal G$ with $c$ channels. Also, we denote the adjacency matrix by $\bfA$, the diagonal degree matrix by $\bfD$,
and the corresponding ones with added self-loops by $\tilde \bfA$ and $\tilde \bfD$, respectively. The graph Laplacian (with self-loops considered) is given by $\tilde{\bfL} = \tilde{\bfD} - \tilde{\bfA}$. 
Lastly, we denote the symmetrically normalized adjacency matrix by $\hat \bfA = \tilde{\bfD}^{-\frac{1}{2}} \tilde{\bfA} \tilde{\bfD}^{-\frac{1}{2}}$ and the symmetrically normalized graph Laplacian by $\hat \bfL = \tilde{\bfD}^{-\frac{1}{2}}\tilde{\bfL}\tilde{\bfD}^{-\frac{1}{2}}$.

\subsection{Overview of the framework}
\label{sec:overview}
Our method is illustrated in Figure \ref{fig:fig1}: we begin by sampling multiple vectors from a random distribution. Once the vectors have been generated, they are processed iteratively by alternating propagation and normalization steps. As a final step, we concatenate all the intermediate results of the process above and feed the result into a GNN as PE. The pseudocode of our method can be found in \Cref{app:alg}.

\textbf{Main components.} 
At the core of our method, there are four main components. 
(1) The first component is the type of RNF used by our method. Two distributions are considered: Standard normal and Rademacher distributions\footnote{A Rademacher random variable has a probability of $0.5$ to get the value $1$ or $-1$.}. Other distributions may also be suitable. 

(2) Our second component is the propagation operator, which governs the information flow between nodes. A propagation operator is denoted by a matrix $\bfS \in \mathbb{R}^{n\times n}$ and can be either predefined or learned. 
(3) 
The third component of our method is the normalization function $\bfN:\mathbb{R}^{n\times k}\rightarrow \mathbb{R}^{n\times k}$, which returns a normalized representation of the current node features. The normalization may take the form of a simple $\ell_2$ normalization of the propagated RNF or a more complex orthonormalization process.

(4) The fourth and last component is the architecture used to process the proposed  PE. MPNNs or graph transformers are natural choices. In addition, we propose a method for processing multiple input trajectories using an equivariant DSS-GNN architecture \cite{bevilacqua2021equivariant} which incorporates set symmetries.  

\textbf{Detailed workflow.} In order to compute our PE, we sample the random features $\bfr \in \mathbb{R}^{n \times k}$ from some joint probability distribution $\mathcal{D}$ and perform $P$ of  the following iterations:

\begin{align}
\label{eq:propApplication}
\bfa^{(0)} =& \; \bfr \nonumber  \\
\hat{\bfa}^{(p)} =& \; \bfS \bfa^{(p-1)}  \nonumber \\
{\bfa}^{(p)} =& \begin{cases}\bfN(\hat{\bfa}^{(p)}) , \  \text{if } \text{mod}(p, w) = 0 \\
\hat{\bfa}^{(p)} , \ \text{otherwise}
\end{cases}.
\end{align}
Here, $\bfS$ is the propagation operator. $\bfN$ is the normalization function mentioned above, and $w$ is a positive integer hyper-parameter that controls the frequency of normalization steps.

We define a feature \textit{trajectory} $\bft$ as the concatenation of the initial random features $\bfr$ and the $P$ steps defined in \cref{eq:propApplication}:
\begin{equation}
    \label{eq:trajectory}
    \bft = \bfr \oplus \bfa^{(1)} \oplus \ldots \oplus \bfa^{(P)} \in \mathbb{R}^{n \times k(P+1)},
\end{equation}
where $\oplus$ denotes channel-wise concatenation.
As discussed in detail in the previous work section, a similar propagation and normalization scheme, starting from the original node features, and using a different normalization function, was recently suggested in \citet{huang2022from}.

In order to enrich our PE, it may be helpful to use \emph{several}  trajectory inputs that originate from different random samples $\bfr_i$. In that case, $\bft$ represents the set of these trajectories:
\begin{equation}
    \label{eq:multiTraj}
    \bft=\{ \bft_1, \ldots , \bft_{B} \},
\end{equation}
with $\bft_b = \bfr_b \oplus \bfa^{(1)}_b \oplus \ldots \oplus \bfa^{(P)}_b \in \mathbb{R}^{n \times k(P+1)}$.

We concatenate and embed our PE $\bft$, and input node features $\bff^{\text{in}} \in \mathbb{R}^{n\times c^{\text{in}}}$, using a linear layer $\bfK_{\text{embed}}$ followed by an optional non-linear activation $\sigma$, as follows:
\begin{equation}
    \label{eq:concatEmbedInputs}
    \bff^{(0)} = \sigma(\bfK_{\text{embed}}(\bff^{\text{in}} \oplus \bft)).
\end{equation}
We regard $\bff^{(0)}$ as the initial node embeddings obtained by our method.
We then utilize several GNN layers to further process those initial node embeddings. 

\subsection{Propagation operators}
\label{sec:operators}
The propagation operator $\bfS$ in \eqref{eq:propApplication} determines the flow of information between nodes in a graph and is represented as a matrix $\bfS \in \mathbb{R}^{n\times n}$. The choice of this operator directly controls the resulting PE.  
We consider two types of propagation operators: (1) \emph{Predefined} propagation operators based on the input graph connectivity, such as the adjacency matrix or graph Laplacian; and (2) \emph{Learned} propagation operators that are generated based on the input graph using a parametric model. We now discuss these two types of operators.

\textbf{Predefined propagation operators.}
We consider two popular predefined operators: the symmetrically normalized adjacency matrix with added self-loop $\hat{\bfA}$ 
(which is widely used in GNNs to propagate features, see e.g.\ \citet{kipf2016semi,wu2019simplifying,chen20simple}) and the analogous symmetrically normalized Laplacian matrix $\hat{\bfL}$. 
When used with RFP, we will see that these operators give access to the eigenvectors of the graph operator. Other predefined propagation operators can also be considered.

\begin{figure}[t]
    \includegraphics[width=\columnwidth]{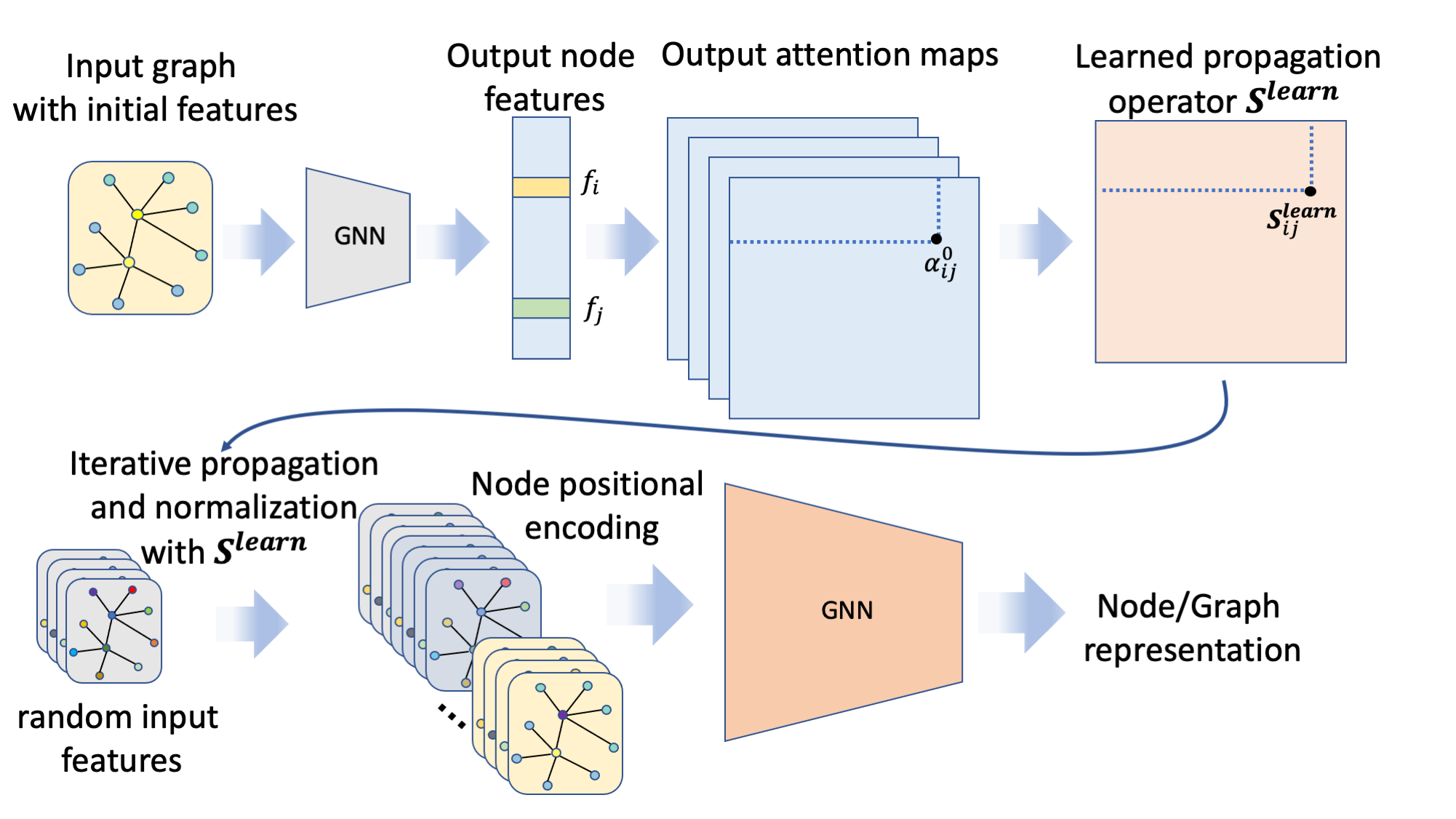}
    \centering
    \caption{Workflow when using a learnable propagation operator. To obtain the propagation operator, we start from the original node features, process them with a GNN and apply an attention mechanism. We then use the resulting operator $\bfS^{\text{learn}}$ to generate our PE by applying it to random features.}
    \label{fig:learn_op}
\end{figure}

\textbf{Learnable propagation operator.} 

We additionally propose to learn the propagation operator. The advantages of using learnable operators are twofold. First, propagation by fixed operators may not produce suitable features for the learning task at hand, so it may be beneficial to tune the propagation operator to the particular learning task. Second, our learnable operators can capture pairwise node interactions that are difficult to model using our previously proposed local operators. To learn the operator, we utilize several GNN layers applied to the original node features, followed by a multi-head self-attention mechanism~\cite{vaswani2017attention} which computes a score for each pair of nodes in the graph (see \cref{fig:learn_op}).
While the predefined operators that we (and previous works)  considered are local, the learnable operator, as we formulate it, is capable of capturing long-distance structural similarities. The reason for this is that the GNN layers before the attention mechanism encode the local neighborhood (both connectivity and node features) around each node, allowing the attention to measure the degree of similarity between those neighborhoods. Due to this, our learnable operator can easily connect two nodes that are very far apart in the graph but have similar local structures. 
Details about the learning procedure are in \Cref{app:learnableOperator}.

\subsection{Normalization function}
\label{sec:normalizationFunction}
A second key component of our RFP framework, is the normalization function. From a practical standpoint, a normalization step helps to control the magnitude of the features, which may grow exponentially with the number of iterations, depending on the propagation operator $\bfS$. Furthermore, as we discuss below, the choice of the normalization function has a significant effect on the feature trajectory $\bft$ and its properties. 
In this paper, we consider two normalization functions: 
a simple channel-wise  $\ell_2$ normalization, and a more sophisticated orthonormalization  described below.

\textbf{$\ell_2$ normalization.}
The first normalization we consider is a simple channel-wise   $\ell_2$ feature normalization, defined as:
\begin{equation}
\label{eq:l2NormNormalization}
    \bfN^{\ell_2}(\hat{\bfa}^{(p)}) = \frac{\hat{\bfa}^{(p)}}{\| \hat{\bfa}^{(p)} \|_2}.
\end{equation}
This normalization function is applied \emph{independently} to each channel of the propagated node features, as described in \cref{eq:propApplication}. Note that using this kind of normalization in tandem with the propagation steps, mimics the Power method iteration \cite{saad2011numerical}, a simple and popular algorithm for computing the dominant eigenvalue and eigenvector of a given diagonalizable matrix $\bfS \in \mathbb{R}^{n\times n}$.
The connection between our propagation and normalization schemes to the power iteration method sheds light on the feature trajectory we use as PE:  given a diagonalizable propagation operator $\bfS$ with some random node features $\bfr \in \mathbb{R}^{n \times k}$, when $w=1$, the $p$-th iteration described in \cref{eq:propApplication} is exactly the $p$-th power method iteration.
This implies that in the limit, i.e., $P \rightarrow \infty$,  under mild assumptions\footnote{e.g.\ if $\bfS$ has a simple spectrum and the initialization has a non-zero projection on the dominant eigenvector.}, all the channels, i.e., columns of $\bfa^{(P)} \in \mathbb{R}^{n \times k}$, will converge (up to sign) to the leading eigenvector of $\bfS$.

\textbf{Orthonormalization.}
The second normalization method we consider is the joint, channel-wise orthonormalization of the node features. Given node features $\hat{\bfa}^{(p)} \in \mathbb{R}^{n \times k}$ at some trajectory iteration $p$, the proposed normalizing function reads:
\begin{equation}
    \label{eq:orthonomalization}
    \bfN^{\text{QR}}(\hat{\bfa}^{(p)}) =  {\bfa}^{(p)} \quad,
\end{equation}
where ${\bfa}^{(p)} \in  \mathbb{R}^{n \times k}$ is such that ${{\bfa}^{(p)}}^{\top} {\bfa}^{(p)} = \bfI_k$  and  the columns of ${\bfa}^{(p)}$  span the column space of $\hat{\bfa}^{(p)}$. 
In practice, several algorithms can realize \cref{eq:orthonomalization}. We use PyTorch's built-in QR decomposition that implements the Householder orthogonalization scheme.
Normalization according to  \cref{eq:orthonomalization} offers a key difference compared to the $\ell_2$ normalization from \cref{eq:l2NormNormalization}:  in the former, the iterations described in \cref{eq:propApplication} take the name of Subspace Iteration Method \cite{saad2011numerical}, a generalization of the power iteration method for computing \emph{several} dominant eigenvectors simultaneously. In Section \ref{sec:theory}, we show that starting with RNF drawn from a continuous distribution, this process almost surely converges to the dominant $k$ eigenvectors of the propagation operator $\bfS$.

\subsection{Architectures}
\label{sec:archs}
We propose to couple our RFP PE with two kinds of architectures. The first is a standard MPNN or transformer that consumes the initial embedding features $\bff^{(0)}$ from \Cref{eq:concatEmbedInputs}, and then performs several message-passing/attention layers based on some backbone layer such as GCN \cite{kipf2016semi}, GraphConv \cite{morris2019weisfeiler}, GIN \cite{xu2019how} or GraphGPS \cite{rampasek2022GPS}.

As we show in \cref{sec:theory}, there are benefits to processing several trajectories. 
 For example, one advantage of using multiple trajectories is that they allow GNNs to easily implement a substructure counting algorithm. However, processing several  trajectories with the GNNs described above is not trivial. One way would be to concatenate them along the feature dimension and apply a standard MPNN as mentioned above. Unfortunately, it is unclear in what order these trajectories should be concatenated: they represent propagations originating from independently sampled random vectors with no canonical order. As a more effective approach to the problem, we suggest incorporating set symmetries by utilizing the DSS-GNN architecture \cite{maron2020learning, bevilacqua2021equivariant}, an architecture that was designed to process unordered sets of graphs.
 
Given multiple trajectories $\{ \bft_1, \ldots, \bft_{B} \}$, $\bft_b \in \mathbb{R}^{n \times k(P+1)}$, we first apply an embedding function $\bfK_{\text{embed}}:\mathbb{R}^{k(P+1)}\rightarrow \mathbb{R}^{d}$ to each trajectory $\bft_b$ independently, and then reshape these embeddings into a $B\times n \times d$ tensor $\textbf{F}$. These tensors can be viewed as a collection of $B$ graphs containing node features based on the above $B$ trajectories and with the same connectivity as the original graph.  Thus, we can process them using the DSS-GNN architecture, which updates the representations of all the nodes for all $B$ graphs simultaneously by using layers of the following form  
\begin{figure}[t]
\includegraphics[width=0.75\columnwidth]{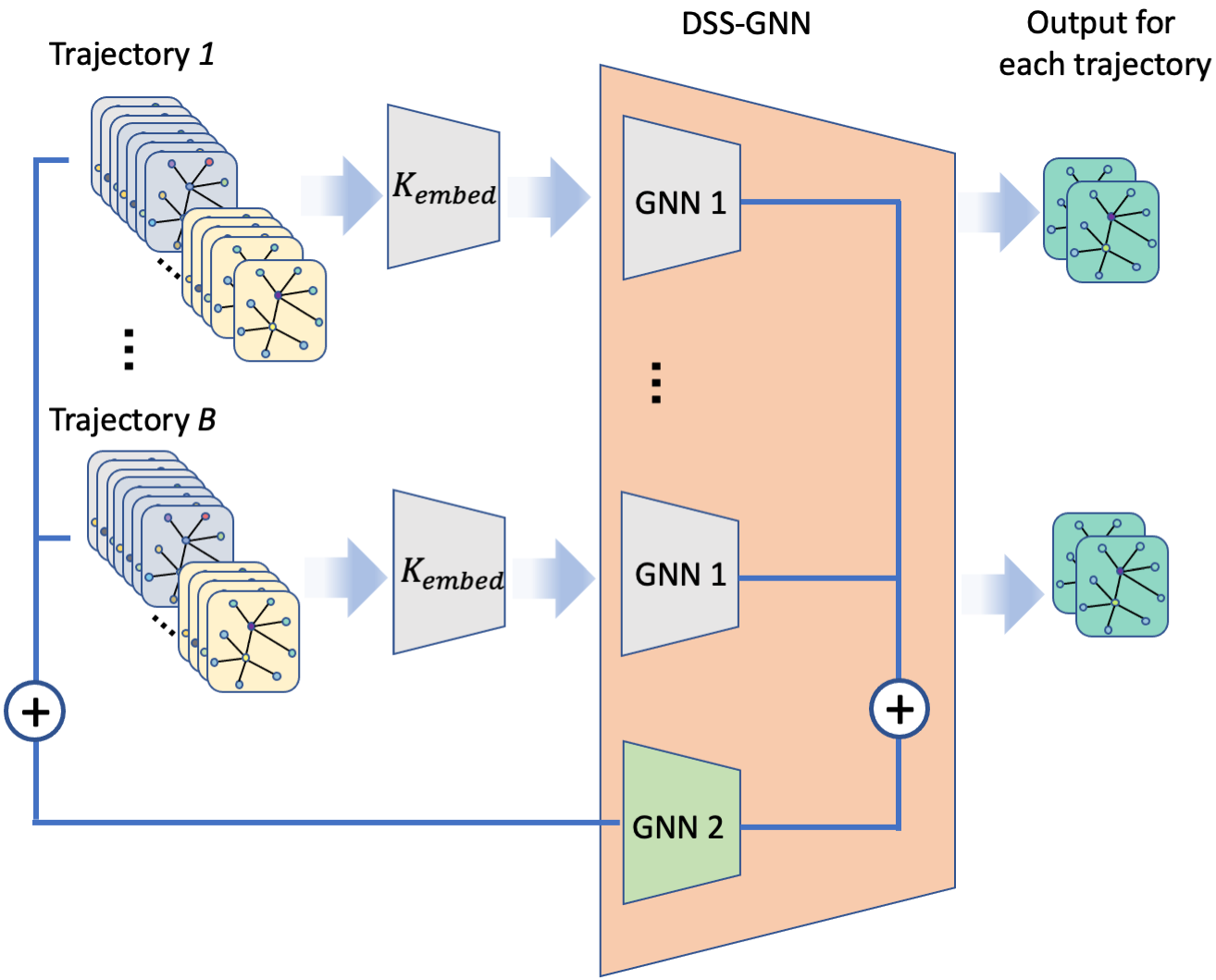}
\centering
\caption{A single layer of our proposed DSS-GNN architecture for processing multiple trajectories in parallel. $\bfK_{\text{embed}}$ and $\text{GNN} \, 1$ are shared across all B trajectories. }
\label{fig:fig2}
\end{figure}

 $$   L(\textbf{F} )_b = \text{GNN}_1(\bff_b) + \text{GNN}_2\left(\sum_{b=1}^{B}\bff_b\right).$$
 $\text{GNN}_1, \text{GNN}_2$ are any GNNs, and $\bff_b$ and $L( \textbf{F} )_b$ refer respectively to the given  and the obtained features of the $b$-th graph in the set. Figure \ref{fig:fig2} illustrates the DSS-GNN layer. In graph-level tasks, after several such layers, a readout function is applied to each graph independently, followed by a Deepsets architecture \cite{zaheer2017deep} applied to the outputs. In node-level tasks, the $B$ dimension in the output tensor is aggregated in a permutation-invariant manner.

\section{Theoretical Analysis}\label{sec:theory}
In this section, we explain the main design choices taken in the development of our RFP framework. Specifically, we elaborate on the following points:
(1) The importance of using multiple random initializations; 
(2) The significance of incorporating early propagation iterations in our RFP PE; and (3) The benefits of using random node features as a starting point for the trajectories over using the original node features. 

All results are proved in \Cref{app:theory}, which also includes an analysis of the time complexity of our approach.

\subsection{Expressiveness} \label{subsec:triangles}

\textbf{Triangle counting.} The ability to count triangles is a fundamental skill of interest for a graph learning model. On the theoretical side, triangles are the simplest substructures standard MPNNs provably cannot count~\citep{chen2020can}; more practically, interesting triangular structures emerge in chemistry, e.g.\ cyclopropane and its derivatives, and social networks, e.g.\ locally clustered communities.
Interestingly, and differently than standard GNNs, our approach is naturally suited for this kind of tasks: it can, in fact, implement the $\textsc{TraceTriangle}_R$ algorithm~\citep{avron2010counting}, a randomized algorithm for the estimation of triangle counts:
\begin{proposition}\label{prop:implementing-trace-triangle}
    Let $\mathcal{G}_n$ be the class of all finite, simple, undirected graphs with $n$ vertices. There exists a choice of hyperparameters $\mathcal{H}$, and weights $\mathcal{W}$, such that our algorithm (ref. \Cref{alg:rfp}), equipped with either a DSS-GNN~\citep{bevilacqua2021equivariant} or GraphConv~\citep{morris2019weisfeiler} downstream architecture, exactly implements the $\textsc{TraceTriangle}_R$ algorithm by \citet{avron2010counting} over $\mathcal{G}_n$.
\end{proposition}
These results may not appear surprising, in light of our RFP framework utilizing random node features, which are known to increase the expressiveness of MPNNs. However, the importance of \Cref{prop:implementing-trace-triangle} lies in showing how the most essential component of our approach, namely the \emph{propagation} of random node features, unlocks specific computational abilities: as we detail in \Cref{cor:approximating-triangle-counts} (\Cref{app:triangles}), our method can compute an approximation of triangle counting without resorting to expensive matrix multiplications. 

\Cref{prop:implementing-trace-triangle} and its proof shed light on two other important aspects of our approach. The first one relates to the information enclosed in early propagation steps: $\textsc{TraceTriangle}_R$ can be implemented by considering only $\bfa^{(1)}, \bfa^{(2)}$, i.e., the features generated by the first two applications of the propagation operator. This suggests that meaningful patterns can already be extracted from initial propagation steps, before reaching eigenvector convergence. These early features, however, are not directly accessible by methods that employ, e.g.\ eigenvector-based PEs, as in \citet{dwivedi2020benchmarkgnns}. The second aspect is related to the fact that our method utilizes multiple trajectories originating from different random starting positions. While this characteristic is essential to implement $\textsc{TraceTriangle}_R$, more generally, it induces a natural alignment with randomized Monte Carlo algorithms and it may also contribute to enhanced performance in certain settings (see \Cref{sec:experiments}).

\textbf{Other substructure features.} \Cref{prop:implementing-trace-triangle} and \Cref{cor:approximating-triangle-counts} can easily be extended so that RFP can approximate the computation of any $\text{trace} (\bfS^P)$, with $\bfS$ a generic (symmetric) operator and $P \in \mathbb{N}$ any integer power, opening up the possibility to extract more sophisticated structural features. For example, it is known that $\text{trace}(\bfA^P)$ counts the number of closed walks of length $P$ and is related to the number of $P$-cycles $c_P$~\citep{perepechko2021the}. A notable case corresponds to that of quadrangles, for which the closed form equation $c_4 = \nicefrac{1}{8} \big ( \text{trace}(\bfA^4) + \text{trace}(\bfA^2) - 2 \sum_{v} \text{deg}^2(v) \big )$ is due to~\citet{harary1971on}.

\textbf{Universal approximation.} \label{subsec:universality} In addition to the results above, it is worth noting that our RFP framework can easily default to the RNI method introduced by~\citet{sato2021random} and subsequently analyzed in its expressiveness by \citet{abboud2020surprising} and \citet{puny2020global}. The latter work derives $(\epsilon, \delta)$ universal approximation properties, which can be shown to be inherited by our RFP algorithm under the same assumptions. We formalize this result in \Cref{app:universal_approximation}.

\subsection{Comparison of random and initial input features} 
\label{subsec:convergence}
While our RFP focuses on the propagation and normalization steps of random node features, the same procedure can instead be applied to the initial, input node features, whenever those are available~\citep{rossi2020sign,huang2022from}. The convergence of the propagation-normalization iterations will then depend on the initial node feature matrix. In particular, in \citet{huang2022from}, given initial node features $X \in \mathbb{R}^{n \times k}$, $X$ must be of full rank in order to guarantee convergence to the $k$ dominant eigenvectors of the propagation operator. 
This highlights two limitations caused by choosing the initial features as the starting point for the propagation-normalization iterations:
(1) Convergence to the dominant eigenvectors is only guaranteed if the matrix of initial node features is of full rank, an assumption that might not always be met; (2) If convergence is reached, the number of dominant eigenvectors is exactly the number of initial input features, and thus cannot be customized to the task at hand. These limitations hinder the use of their method when learning graphs with no features or with low dimensional features, which often appear in graph classification tasks.
On the contrary, the usage of random node features allows decoupling the number of eigenvectors to be obtained at convergence from the number of given initial features. We formalize that property in the following proposition, which we prove in \Cref{app:convergence}.
\begin{proposition}\label{theo:convergence}
For any $k \leq n$, let $X = [X_1, \dots, X_k] \in \mathbb{R}^{n \times k}$ be a concatenation of continuous i.i.d random variables, $X_i \in \mathbb{R}^n$ with probability density function $f$. Let $\mathbf{S} \in \mathbb{R}^{n \times n}$ be a graph operator, and denote by $\lambda_1, \dots, \lambda_k$ the $k$ dominant eigenvalues of $\mathbf{S}$ in decreasing order of magnitude. Assume $\mathbf{S}$ to be symmetric and that  $\lvert \lambda_i \rvert > \lvert \lambda_{i+1} \rvert$, $1 \le i \le k$.
Then, the Subspace Iteration algorithm, and hence the iteration of our RFP PE scheme, converges (up to sign) to the $k$-dominant eigenvectors of $\mathbf{S}$.
\end{proposition}

\section{Experiments}
\label{sec:experiments}
We conduct a comprehensive set of experiments to answer four main questions: (1) Does RFP outperform its natural baselines, namely RNF and spectral PE? (2) Is there a benefit to using one of our proposed architectures over the other? (3) How do different operators compare? Specifically, does the learnable operator improve performance? (4) How do different normalization schemes compare? Full details on experiments as well as additional results are in \cref{app:exp}.

\textbf{Baselines.} RFP is compared with a number of popular and recent methods. Specifically, we focus on the comparison of RNF  \cite{abboud2020surprising}, and Laplacian PE  \cite{dwivedi2020benchmarkgnns}.  When using Laplacian eigenvectors only as a baseline (denoted by $\hat{\bfL}$ EIGVECS$^\dagger$) we use the eigenvectors corresponding to the smallest eigenvalues, as it is standard in the literature \cite{dwivedi2020benchmarkgnns}. For other baselines that involve PE based on other operators, we use the eigenvectors corresponding to the largest eigenvalues, to be consistent with the power method and RFP and allow easy comparison.

\subsection{Graph-level tasks}
\label{sec:graphLevelTasks}

\textbf{Real World Data.}
We consider both the OGBG-MOLHIV \citep{Hu2020Strategies} and the ZINC-12k (500k budget) \citep{ZINCdataset,gomez2018auto,dwivedi2020benchmarkgnns} molecular datasets, where the tasks consist of predicting whether an input molecule inhibits HIV replication in the former, and the value of the constrained solubility in the latter. As a natural baseline, here, we also consider RWPE \cite{dwivedi2022graph} that utilizes powers of the random walk operator as positional encoding. In all baselines and RFP variants, we use GINE \cite{hu2020ogb} as a backbone GNN, that extends GIN \cite{xu2019how} by considering both node and edge features. Results are reported in \Cref{table:graphTasks_realData}. We observe that: (1)
 RFP significantly outperforms the natural baselines, such as Laplacian ($\hat{\bfL}$) and adjacency matrix ($\hat{\bfA}$) based PE, as well as RNF and their combination; (2) The RFP - DSS variant is consistently proven to be beneficial with respect to the GNN version; (3) Learning the propagation operator $\bfS^{\text{learn}}$ (which we couple with a DSS architecture) further yields better results.
    Overall,  our RFP method can outperform a variety of methods, including provably expressive GNNs such as \citet{corso2020principal,beaini2021directional}, as well as HIMP \cite{fey2020heirarchical} and GSN \citep{bouritsas2022improving}, that explicitly use domain-specific knowledge (and being outperformed only by CIN \citep{bodnar2021CW}). In \cref{app:ablations} we show that further improvements can be obtained by combining our RFP with domain-aware structural encoding methods.
  \begin{table}[t]
    \caption{ZINC-12k and OGBG-MOLHIV datasets. Backbone is GINE for all baselines and RFP variants.
     $\dagger$ denotes eigenvectors corresponding to eigenvalues in ascending order of magnitude. The top three models are highlighted by \first{First}, \second{Second}, \third{Third}.}\label{table:graphTasks_realData} 
 \footnotesize
  \center{
 \resizebox{1.0\linewidth}{!}{
  \begin{tabular}{lcc}
    \toprule
    \multirow{2}{*}{Method} & ZINC-12K & OGBG-MOLHIV \\
     & MAE $\downarrow$ & ROC-AUC(\%) $\uparrow$\\
        \midrule
    \textbf{Standard GNNs} & & \\
    $\,$ GCN  & 0.321 $\pm$ 0.009 & 76.06 $\pm$ 0.97  \\
    $\,$ GIN &  0.163 $\pm$ 0.004  & 75.58 $\pm$ 1.40 \\
    $\,$ PNA &  0.133 $\pm$ 0.011 & 79.05 $\pm$ 1.32   \\
    $\,$ DGN & -- & 79.70 $\pm$ 0.97 \\
            \midrule
    \textbf{Domain-aware} & & \\
    $\,$ HIMP & -- &  78.80 $\pm$ 0.82 \\
    $\,$ GSN & \second{0.101 $\pm$ 0.010} &  \third{80.39 $\pm$ 0.90}  \\
    $\,$ CIN & \first{0.079 $\pm$ 0.006} &  \first{80.94 $\pm$ 0.57}  \\
    \midrule
    \textbf{Natural baselines} & & \\
    $\,$ $\hat{\bfL}$ EIGVECS $^\dagger$ & 0.1557 $\pm$ 0.012  &  77.88 $\pm$ 1.82 \\
    $\,$ $\hat{\bfA}$ EIGVECS & 0.1501 $\pm$ 0.018  & 78.13 $\pm$ 1.61\\
    $\,$ $\hat{\bfL}, \hat{\bfA}$ EIGVECS & 0.1434 $\pm$ 0.015  & 78.30 $\pm$ 1.47 \\
    $\,$ RNF & 0.1621 $\pm$ 0.014 & 75.98 $\pm$ 1.63 \\
    $\,$ RNF \& $\hat{\bfL}, \hat{\bfA}$ EIGVECS &  0.1408 $\pm$ 0.020 &  78.96 $\pm$ 1.33 \\ 
    $\,$ RWPE &  0.1279 $\pm$ 0.005 & 	78.62 $\pm$ 1.13\\
    \midrule
    \textbf{RFP (ours)}& & \\
    $\,$ RFP - ${\ell_2}$ - ${\hat{\bfL}}, {\hat{\bfA}}$ & 0.1368 $\pm$ 0.010  & 77.91 $\pm$ 1.43 \\ 
    $\,$ RFP - QR - ${\hat{\bfL}}, {\hat{\bfA}}$ & 0.1152 $\pm$ 0.006  & 79.83 $\pm$ 1.16 \\
     $\,$ RFP - QR - ${\hat{\bfL}}, {\hat{\bfA}}, \bfS^{\text{learn}}$  & 0.1143 $\pm$ 0.008 & 79.94 $\pm$ 1.51 \\
     $\,$ RFP - QR - ${\hat{\bfL}}, {\hat{\bfA}}$ -DSS & 0.1117 $\pm$ 0.009  & 80.53 $\pm$ 1.04\\
     $\,$ RFP - QR - ${\hat{\bfL}}, {\hat{\bfA}, \bfS^{\text{learn}}}$ -DSS  & \third{0.1106 $\pm$ 0.012} & \second{80.58 $\pm$ 1.21}\\
    \bottomrule
  \end{tabular}}}
  \end{table}

\begin{table*}[t]
    \caption{Node classification accuracy ($ \%$) $\uparrow$. \textsuperscript{*} denotes the best result out of several variants.  An MLP backbone is used for Cornell, Texas, and Wisconsin, and a GCN backbone for the rest of the datasets. The top three models are highlighted by \first{First}, \second{Second}, \third{Third}.}
\label{table:nodeClassification} %
  \center{
  \footnotesize
  \resizebox{\linewidth}{!}{
  \begin{tabular}{lcccccc|ccc}
    \toprule
    Method & Squirrel & Film &  Cham. & Corn. & Texas & Wisc. & Citeseer & Pubmed & Cora \\
    Homophily & 0.22 & 0.22 & 0.23 & 0.30  & 0.11 & 0.21 & 0.74 & 0.80 & 0.81 \\
        \midrule
    \textbf{Standard baselines} & & \\
    $\,$ MLP & 28.77  & 36.41 & 46.24 & 80.89 & 80.81 & 83.88 &  74.16 & 87.05 & 75.59  \\
    $\,$ GCN & 23.96 & 26.86 & 28.18 & 52.70 & 52.16 & 48.92 & 73.68 & 88.13 & 85.77  \\
      \midrule
    \textbf{Node classification designated architectures} & & \\
    $\,$ WRGAT & 48.85 & 36.53 & 65.24 &  81.62 & 83.62 & 86.98 & 76.81 & 88.52 & \first{88.20} \\
    $\,$ GGCN  & \second{55.17} & \first{37.81} &  \second{71.14} & \third{85.68}  & \third{84.86}  &  86.86 & \first{77.14} & \third{89.15} & 87.95 \\
    $\,$ H2GCN  & 36.48 & 35.70 & 60.11 & 82.70  & \third{84.86}  &  \second{87.65} & \third{77.11} & \second{89.49} & 87.87 \\
    $\,$ GPRGNN & 31.61 & 34.63 & 46.58 & 80.27 & 78.38 & 82.94 & \second{77.13} & 87.54 & \third{87.95} \\
    $\,$ G\textsuperscript{2}-GCN & 46.48 & 37.09 & 55.83 & \second{86.49} & 84.86 & 87.06 & -- & -- & -- \\
    $\,$ G\textsuperscript{2}-GraphSAGE & \first{64.26} &  \second{37.14} & \first{71.40} & 86.22 &  \first{87.57} & \first{87.84} & -- & -- & --\\
    \midrule
    \textbf{Natural baselines} & & \\
        $\,$ $\hat{\bfL}$ 64 EIGVECS & 35.85 & 34.06 & 50.87 & 74.10 & 74.03 & 78.17 & 74.42 &  88.17  &  86.31  \\
        
        $\,$ $\hat{\bfA}$ 64 EIGVECS & 35.15 & 34.05 & 51.44 & 75.40 & 72.97 & 77.25 & 75.01 & 88.41 & 86.37  \\
        $\,$ $\hat{\bfA}$ ALL EIGVECS & 34.78 & 34.32 & 50.99 & 72.70  & 73.24 & 80.18 & 73.89 & 88.03 & 85.91   \\
        $\,$ RNF &  34.40 & 35.34 & 51.20 &  78.90 & 80.03 &  81.93 &  73.70 & 88.61 & 85.43   \\
        $\,$ RNF \& 64
        $\hat{\bfA}$ EIGVECS & 44.89 & 34.68 & 58.21 & 79.46 & 78.37 & 82.25 & 75.29 & 88.96 & 85.67\\
    \midrule
    \textbf{Propagation-based PE} & & \\
$\,$ SGC & 37.07 & -- & 55.11 & 75.68 & 75.68 & 75.29 & 73.39 & -- & 84.89\\
$\,$ SIGN & 40.97 & -- & 60.11 & 76.76 & 75.14 & 78.43 & 73.27 & -- & 83.92 \\
$\,$ PowerEmbed\textsuperscript{*} & 53.53 &  -- & 64.98 & 78.30 &  79.19 & 78.43 & 73.27 & -- & 85.03 \\
        $\,$ NFP - QR - ${\hat{\bfA}}$ & 46.85 & 27.72 & 61.25 & 65.40 & 78.23 & 62.97 & 73.89 & 88.86 & 85.93\\
    \midrule
    \textbf{RFP (ours)}& & \\
       $\,$ RFP - ${\ell_2}$ - ${\hat{\bfA}}$  & 48.89 & 35.11 & 62.13 & 82.13 & 81.18 & 81.66 & 75.59 & 88.62 & 87.10 \\
      $\,$ RFP - QR - ${\hat{\bfA}}$  & 54.28 & 36.87 & 65.98 & 85.13 & 84.74 & 86.08 & 76.48 & \first{90.01} & 87.99 \\
     $\,$ RFP - QR - ${\hat{\bfA}}$ -DSS  & 54.78 & 36.93 & 66.16 & 85.69 & 85.10 & 86.35 & 76.92 & 90.00 & \second{88.09} \\
  $\,$ RFP - QR - $\hat{\bfL}, \hat{\bfA}, \bfS^{\text{learn}}$ - DSS &  \third{54.97} & \third{37.12} & \third{66.81} & \first{86.74} & \second{87.13} & \third{87.05}  & 76.53 & 89.96 &  87.99 \\
    \bottomrule
  \end{tabular}}}
\end{table*}

\paragraph{Synthetic Data.}
\label{sec:expressivenesTask}
To validate our theoretical analysis on substructure counts, we conduct experiments on popular datasets \citep{chen2020can,corso2020principal}. Coherently with our theoretical claims, the results (deferred to \Cref{app:exp}) show: (1) RFP attains larger expressiveness than standard MPNNs; (2) early propagation steps and the inclusion of multiple trajectories result in a lower test error.

\subsection{Node-level tasks}
\label{sec:nodeClassification}
To further highlight the benefits of RFP, we consider node-level tasks. We report results on 9 datasets consisting of both  homophilic graphs (Cora \cite{mccallum2000automating}, Citeseer \cite{sen2008collective}, and Pubmed \cite{namata2012query}) and heterophilic graphs (Squirrel, Film, and Cham. \cite{musae}, and Cornell, Texas, and Wisconsin \cite{Pei2020Geom-GCN:}) \footnote{The definition of homophilic and heterophilic datasets can be found in \citet{Pei2020Geom-GCN:}.}. In all datasets, we use the standard 10 splits from \citet{Pei2020Geom-GCN:}.  
The results are presented in \Cref{table:nodeClassification}, which reveals several interesting findings. First, our RFP - QR variants consistently improve the baselines GCN and MLP. Furthermore, it offers higher accuracy than the natural baselines  RNF  and Laplacian PE, in many cases by a large margin. Interestingly, we also find that using operator (e.g., $\hat{\bfA}$) eigenvectors as PE performs worse than the MLP baseline on Cornell, Texas, and Wisconsin datasets.  RFP also consistently outperforms propagation-based PEs such as SGC \cite{wu2019simplifying}, SIGN \cite{rossi2020sign}, PowerEmbed \citep{huang2022from}, as well as Node Feature Propagation (NFP), a variant of our approach employing our propagation normalization mechanism to the original input node features rather than random node features. This observation shows that the choice of using random node features as a starting point to our trajectory is not only justified theoretically but also provides better results in practice.
Lastly, we see that our RFP is in line with, or better than, recent methods that are node classification designated, such as H2GCN  \cite{zhu2020beyondhomophily_h2gcn}, GGCN \cite{yan2021two} and GPRGNN \cite{chien2021adaptive}. On most datasets, our RFP also offers similar performance to the recent G\textsuperscript{2} \cite{rusch2022gradient}; larger gaps are observed on Squirrel and Cham., where, however, their superior performance seems to be mostly driven by a different choice of the GNN layer (GraphSAGE). We note that our RFP can also be applied to large datasets, as opposed to the exact computation of the eigenvectors that is required in other PE methods. Thus, in \cref{app:additional_expr}, we report the results obtained with our RFP on the large ogbn-arxiv dataset as well as the required computational resources.

\subsection{Ablation study}
\label{sec:ablation}
We conduct an ablation study to investigate the effect of the number of propagation-normalization steps $P$ as well as the number of chosen eigenvectors $k$. We report our results in \Cref{app:ablations}. In general, we found that increasing the number of steps and eigenvectors resulted in better scores.

\subsection{Discussion}

We have experimented with RFP on node-level and graph-level tasks.  RFP consistently outperforms  its natural baselines, in some cases by a wide margin. 
These results demonstrate the effectiveness of the main choices we made when designing our approach, including the fact that the full trajectory is informative and significant, and that starting the propagation from RNF is more effective than using initial features. 
Importantly, we also found orthonormalization ($\bfN^{\text{QR}}$) to consistently outperform $\ell_2$ normalization ($\bfN^{\ell_2}$) in all tasks.

In general, our method obtains competitive results on both node and graph tasks, while being outperformed only by domain-specific architectures (in the case of molecules) or architectures that are tailored for a specific task (in the case of node classification).

\section{Conclusion}
\label{sec:conclusion}
Our paper introduces RFP, a new family of PEs for graph learning, which combines two existing techniques, RNF and spectral PE, and has demonstrated favorable theoretical properties as well as better practical performance over these and even recent, task-specific techniques. We believe this work opens interesting future research directions, including exploring new operators and normalizations, designing novel parametric models for learning operators, and applying RFP for link prediction.

\section*{Acknowledgements}
The research reported in this paper was supported by the Israeli
Council for Higher Education (CHE) via the Data Science
Research Center, Ben-Gurion University of the Negev, Israel. ME is supported by Kreitman High-tech scholarship.

\bibliography{main}
\bibliographystyle{icml2023}

\newpage
\appendix
\onecolumn

\begin{table}[t]
    \caption{Graph Counting dataset MAE $\downarrow$. Backbone is GIN for all baselines and RFP variants. $\dagger$ denotes eigenvectors corresponding to eigenvalues in ascending order of magnitude.
    The top three models are highlighted by \first{First}, \second{Second}, \third{Third}. }
  \label{table:graphCounting}
\vskip 0.1in
\center{
    \footnotesize
  \begin{tabular}{lcccc}
    \toprule
    Method & Triangle & Tailed-tri & Star & Cyc4 \\
        \midrule
    \textbf{Standard models} & & \\
    $\,$ GCN  &  0.4186 & 0.3248 &0.1798 & 0.2822 \\
    $\,$ GIN &  0.3569 &0.2373 &0.0224& 0.2185
 \\
    $\,$ PNA &  0.3532 & 0.2648 & 0.1278 &0.2430    \\
    $\,$ PPGN & \second{0.0089} & \second{0.0096} & 0.0148 & \first{0.0090} \\
    \midrule
    \textbf{Subgraph GNNs} & & \\
    $\,$ GNN-AK & 0.0934 & 0.0751 & 0.0168 & 0.0726 \\
    $\,$ GNN-AK-CTX  & 0.0885 & 0.0696 & 0.0162 & 0.0668 \\
    $\,$ GNN-AK+ & \third{0.0123} &  \third{0.0112} & 0.0150 & \third{0.0126} \\
        $\,$ DSS-GNN (EGO+) & 0.0450 &	0.0393 &	0.0129 &	0.0550 \\

    $\,$ SUN (EGO+) & \first{0.0079} & \first{0.0080} & \first{0.0064} & \second{0.0105} \\
    \midrule
    \textbf{Natural baselines} & & \\
    $\,$ $\hat{\bfL}$ EIGVECS $^\dagger$ & 0.1681  & 0.1728 & 0.0168 & 0.1572 \\
    $\,$ $\hat{\bfA}$ EIGVECS & 0.1904 & 0.1714 & 0.0140 &  0.1533 \\
    $\,$ $\hat{\bfL}, \hat{\bfA}$ EIGVECS & 0.1694 & 0.1680 & 0.0121 & 0.1506\\
    $\,$ RNF  & 0.1775 & 0.1807 & 0.0189 &  0.1516 \\
    $\,$ RNF \& $\hat{\bfL}, \hat{\bfA}$ EIGVECS  & 0.1512 & 0.1503 & \third{0.0120} & 0.1497  \\ 
    \midrule
    \textbf{RFP (ours)}& & \\
     $\,$ RFP - ${\ell_2}$ - ${\hat{\bfL}}, {\hat{\bfA}}$ - GINE & 0.1521 & 0.1573 & 0.0177 & 0.1439  \\ 
    $\,$ RFP - QR - ${\hat{\bfL}}, {\hat{\bfA}}$ - GINE & 0.0923 & 0.0826 & 0.0118 & 0.0991 \\
    $\,$ RFP - QR - ${\hat{\bfL}}, {\hat{\bfA}}$ - GINE-DSS & 0.0820 & 0.0764  & \second{0.0072} & 0.0715 \\
    \bottomrule
  \end{tabular}}
  \end{table}

\section{Experimental Results}
\label{app:exp}

\begin{table}[t]
    \caption{Ablation study on the effect of the chosen number of eigenvectors $k$ in RFP with fixed number of propagations $P=16$ on ZINC-12k and OGBG-MOLHIV datasets.}
  \label{table:zinc12k_molhiv_ablation}   
  \vskip 0.1in
  \center{
  \footnotesize
  \begin{tabular}{lc|cc}
    \toprule
    \multirow{2}{*}{Method} & \multirow{2}{*}{k} &  ZINC-12K & OGBG-MOLHIV  \\
    & & MAE $\downarrow$ & ROC-AUC (\%) $\uparrow$ \\
        \midrule
    \textbf{GINE Backbone} \\
    $\,$ $\,$RFP - ${\ell_2}$ - ${\hat{\bfL}}, {\tilde{\bfA}}$ & 64 &  0.1368 $\pm$ 0.010  & 77.91 $\pm$ 1.43\\ 
   $\,$ $\,$RFP - QR - ${\hat{\bfL}}, {\hat{\bfA}}$ & 4 &  0.1305 $\pm$ 0.007 &  78.26 $\pm$ 1.51 \\
    $\,$ $\,$RFP - QR - ${\hat{\bfL}}, {\hat{\bfA}}$ & 8 & 0.1223 $\pm$ 0.009 &  79.43 $\pm$ 1.37  \\
    $\,$ $\,$RFP - QR - ${\hat{\bfL}}, {\hat{\bfA}}$ & 16 & 0.1152 $\pm$ 0.006 &  79.83 $\pm$ 1.16 \\
    $\,$ $\,$RFP - QR - ${\hat{\bfL}}, {\hat{\bfA}}, \bfS^{\text{learn}}$ &  16& 0.1143 $\pm$ 0.008 & 80.29 $\pm$ 1.01\\
    $\,$ $\,$RFP - QR - ${\hat{\bfL}}, {\hat{\bfA}}$ - DSS & 16  & 0.1117 $\pm$ 0.009 &  80.53 $\pm$ 1.04\\
    $\,$ $\,$RFP - QR - ${\hat{\bfL}}, {\hat{\bfA}, \bfS^{\text{learn}}}$ - DSS  & 16  & 0.1106 $\pm$ 0.012 &  80.58 $\pm$ 1.21 \\
    \midrule 
    \textbf{GraphGPS Backbone} \\
    $\,$ $\,$RFP - ${\ell_2}$ - ${\hat{\bfL}}, {\hat{\bfA}}$ & 64 & 0.1322 $\pm$ 0.014 &  77.98 $\pm$ $1.57$\\ 
    $\,$ $\,$RFP - QR - ${\hat{\bfL}}, {\hat{\bfA}}$ & 16 & 0.1100 $\pm$ 0.005 & 79.95 $\pm$ 1.41\\
    \midrule
    \textbf{Combining RFP with domain-aware PE} \\
        $\,$ $\,$RFP+GSN - QR - ${\hat{\bfL}}, {\hat{\bfA}}$ & 16 & 0.0734 $\pm$ 0.007 & 80.07 $\pm$ 1.19\\

    \bottomrule
  \end{tabular}}
  \end{table}

\begin{table*}
    \caption{Ablation study on the effect of the chosen number of eigenvectors $k$ in RFP with fixed number of propagations $P=16$ on the Graph Counting dataset test MAE $\downarrow$.
    GIN is used as backbone. }
\label{table:graphCounting_ablation}\vskip 0.1in
\center{
\footnotesize
  \begin{tabular}{lc|cccc}
    \toprule
    Method & k & Triangle & Tailed-tri & Star & Cyc4 \\
        \midrule
    RFP\textsubscript{$\ell_2$} & 16  & 0.1521 & 0.1573 & 0.0177 & 0.1439  \\ 
    RFP - QR - ${\hat{\bfL}}, {\hat{\bfA}}$  & 4 & 0.1322 & 0.01289 & 0.0126 & 0.1401 \\
    RFP - QR - ${\hat{\bfL}}, {\hat{\bfA}}$  & 8 & 0.1231 & 0.1143 & 0.0119 & 0.1170 \\
    RFP - QR - ${\hat{\bfL}}, {\hat{\bfA}}$ & 16 & 0.0923 & 0.0826 & 0.0118 & 0.0991 \\
    RFP - QR - ${\hat{\bfL}}, {\hat{\bfA}}$ - DSS & 16 & 0.0820 & 0.0764  & 0.0072 & 0.0715 \\
    \bottomrule
  \end{tabular}}
  \end{table*}

\begin{table*}[t]
  \begin{minipage}[t]{.99\linewidth}
    \caption{Node classification accuracy ($ \%$) $\uparrow$ on homophilic and heterophilic datasets.} 
    \label{tab:ncAblation}
    \vskip 0.1in
  \center{
  \scriptsize
  \begin{tabular}{lccc|cccccc|ccc}
    \toprule
    Method & OP & k (vecs) &  \# props &  Squirrel & Film &  Cham. & Corn. & Texas & Wisc. & Cora & Cite & Pub \\
    Homophily & -- & -- & -- &   0.22 & 0.22 & 0.23 & 0.30  & 0.11 & 0.21 & 0.81 & 0.74 & 0.80 \\
        \midrule
    \midrule
    EIGVEC-PE & $\hat{\bfL}$ & 64 & -- & 35.85 & 34.06 & 50.87 & 74.10 & 74.03 & 78.17 & 86.31  &  74.42 & 88.17 \\ 
    EIGVEC-PE & $\hat{\bfA}$ (TopAbs) & 64 & -- & 35.15 & 34.05 & 51.44 & 75.40 & 72.97 & 77.25 & 86.37 & 75.01 & 88.41 \\
    EIGVEC-PE ALL & -- & $|\mathcal{V}|$ & -- & 34.78 & 34.32 & 50.99 & 72.70  & 73.24 & 80.18 & 85.91 & 73.89 & 88.03 \\
    \midrule
     RNF & -- & 64 & 0 & 34.40 & 35.34 & 51.20 &  78.90 & 80.03 &  81.93 & 85.43 & 73.70 & 88.61 \\

      RNF & -- & $|\mathcal{V}|$ & 0 & 34.36 & 34.74  &  53.28&  79.30 & 80.00 & 81.97  & 84.56  & 72.29  & 88.51  \\
      RNF \& $\hat{\bfA}$ EIGVECS & $\hat{\bfA}$ & 64 & -- &  44.89 & 34.68 & 58.21 & 79.46 & 78.37 & 82.25 & 
 85.67 & 75.29 & 88.96   \\
     \midrule
         NFP - QR &  ${\hat{\bfA}}$ & 64 & 16 & 46.85 & 27.72 & 61.25 & 65.40 & 78.23 & 62.97 & 73.89 & 88.86 & 85.93\\
    \midrule
        $\,$RFP - ${\ell_2}$ & $\hat{\bfL}$ & 1 & 16 & 48.89  & 35.11 & 62.13 & 82.13  &  81.18  & 81.66 & 87.10 & 75.59 & 88.62  \\ 
\midrule
        $\,$RFP - QR & $\hat{\bfA}$ & $|\mathcal{V}|$ & 16 & 54.23 & 36.61  &  64.77 & 85.07  & 85.41 &  86.27 & 85.68 & 75.16 & 89.00  \\

\midrule
        $\,$RFP - QR & $\hat{\bfA}$ & 64 & 2 & 53.87 & 36.40 & 64.50   & 83.78  & 82.70   & 
 83.13 & 86.26 & 75.22 & 88.68 \\
        $\,$RFP - QR & $\hat{\bfA}$ & 64 & 4 & 54.11 & 36.44 & 65.16 & 84.59 & 83.78 & 84.90 & 86.67 & 75.37 & 88.76 \\
        $\,$RFP - QR & $\hat{\bfA}$ & 64 & 8 & 54.34 & 36.56 & 65.83 & 84.86 & 84.05 & 85.88 & 87.35 & 75.94 &  89.95 \\
        $\,$RFP - QR & $\hat{\bfA}$ & 64 & 16 &  54.28 & 36.87 & 65.98 & 85.13 & 84.32
 & 86.08 & 87.99 & 76.40 & 90.01 \\
        $\,$RFP - QR & $\hat{\bfA}$ & 64 & 32 & 54.25 & 36.92 & 65.80 & 85.04 & 84.59 & 85.88  & 87.81 & 76.43 & 90.03 \\
        \midrule
        $\,$RFP - QR & $\hat{\bfA}$ & 4 & 16 & 51.96  & 35.88 & 64.01 & 83.24 & 82.70 & 82.16 & 86.93 & 76.11 & 89.66  \\
        $\,$RFP - QR & $\hat{\bfA}$ & 8 & 16 & 52.72 & 35.92 & 64.63 & 83.78 & 83.51  & 82.43 & 87.32 & 76.27 & 89.93\\
        $\,$RFP - QR & $\hat{\bfA}$ & 16 & 16 & 53.87 & 36.51 & 65.11 &  84.59& 83.78 & 82.97 & 87.65 & 76.39 & 90.00 \\
        $\,$RFP - QR & $\hat{\bfA}$ & 32 & 16 & 54.33 & 36.42& 65.66 &  84.86 & 84.74 & 84.05 & 87.84  & 76.48 & 89.97\\
        $\,$RFP - QR & $\hat{\bfA}$ & 64 & 16 & 54.28 & 36.87 & 65.98 &  85.13 & 84.32
 & 86.08 & 87.99  & 76.40 & 90.01 \\
        $\,$RFP - QR & $\hat{\bfA}$ & 128 & 16 & 54.40 & 36.70 & 65.90 & 85.07 & 84.59
  & 86.12 & 87.96 & 76.42 & 89.83\\
    \midrule
        $\,$RFP - QR & $\hat{\bfL}, \hat{\bfA}, \bfS^{\text{learn}}$ & 64/32 & 16/8 & 54.56 & 36.95  & 66.34 & 86.40 & 86.72 & 86.88 & 87.92 & 76.41 & 89.63 \\
    \midrule
    $\,$RFP - QR - DSS  & $\hat{\bfA}$ &  64/32 & 16/8 & 54.78 & 36.93 & 66.16 & 85.69 & 85.10 & 86.35 & 88.09 & 76.92 & 90.00 \\
    $\,$RFP - QR - DSS  & $\hat{\bfL}, \hat{\bfA}, \bfS^{\text{learn}}$ &  64/32 & 16/8 & 54.97 & 37.12 & 66.81 & 86.74 & 87.13 & 87.05 & 87.99 & 76.53 & 89.96 \\

    \bottomrule
  \end{tabular}}
\end{minipage}
\end{table*}

\subsection{Experimental settings}
\label{app:experimentalSettings}
Our code is implemented using PyTorch \cite{pytorch} and PyTorch-Geometric \cite{pyg2019}, and all our experiments are run on Nvidia RTX3090 GPUs with 24GB of memory.

\paragraph{Natural baselines.} 
RFP is compared with a number of popular and recent methods. In particular, we focus on the comparison of RNF  \cite{abboud2020surprising}, Laplacian PE  \cite{dwivedi2020benchmarkgnns} and PowerEmbed \cite{huang2022from} which is arguably the closest approach to RFP. 
Additionally, we introduce further baselines based on our RFP method, such as including both the RFP and eigenvectors of the propagation operators without a complete RFP trajectory in order to demonstrate the importance of considering the complete trajectory.
{ For the baseline case of Laplacian eigenvectors only (denoted by $\hat{\bfL}$ EIGVECS$^\dagger$) we use the eigenvectors corresponding to the smallest eigenvalues, in order to be consistent with the literature \cite{dwivedi2020benchmarkgnns}. In the rest of the experiments that involve any other operator PE, we use the eigenvectors corresponding to the largest eigenvalues, to be consistent with the power method and our RFP.}

\paragraph{Hyperparameters.}
We now list the hyperparameters in our experiments. We denote the learning rate by $\rm{lr}$, weight decay by $\rm{wd}$, and dropout probability by $\rm{drop}$. The number of layers is denoted by $L$ and the number of hidden channels by $c$.  
Besides those standard hyperparameters, our RFP requires two main hyperparameters: the number of propagations $P$, and the number of eigenvectors to be estimated $k$. In the case of DSS variants, we also choose the number of trajectories $B$.
The hyperparameters values were determined using a grid search, and the considered values for those hyperparameters are reported in the following:
${\rm{lr}} \in \{1e-4, 1e-3, 1e-2, 1e-1\} , \ {\rm{wd}} \in \{0, 1e-6, 1e-5, 1e-4,1e-3 \}, \ {\rm{drop}} \in \{0, 0.5\} ,\ L \in \{2,4,8,16\}, \ c \in \{ 64,128,256\}, \ P \in \{ 2,4,8,16,32 \} , \  k \in \{ 4,8,16,32,64 \}, \  B \in \{5, 10\}$.  Additional parameters, such as the number of epochs, the stopping criteria, and the evaluation procedure are as prescribed by the corresponding dataset (see e.g.\ \citet{Pei2020Geom-GCN:} and \citet{frasca2022understanding} for the node- and graph-level tasks, respectively). In the ZINC-12k dataset, we further make sure to remain within the prescribed 500k parameter budget.
Finally, unless otherwise specified, we use a standard normal distribution to draw the initial random node features.

\subsection{Additional experiments: the synthetic Graph Counting dataset}
Besides real-world datasets which are presented in the main text, we also experiment with a synthetic dataset that is typically used for assessing the expressiveness of GNN~\citep{chen2020can,corso2020principal}. The objective is to count the number of sub-structures present in a graph, such as triangles or cycles of length 4. To be aligned with our theoretical results in \cref{sec:theory} and \cref{app:theory}, we use the Rademacher distribution to draw random node features to be used in our RFP. In our experiments, we found similar performance when using normally distributed random node features.   Furthermore, we also used the Rademacher distribution for all the other baselines besides the RNF, where we consider a normal distribution to be consistent with RNI \cite{abboud2020surprising}. We follow the experimental settings proposed in \citet{zhao2022from}, and we minimize the mean absolute error (MAE) of the prediction.  As baseline methods we consider several GNNs such as GIN \cite{xu2019how}, PPGN \cite{maron2019provably}, PNA \cite{corso2020principal}, GCN \cite{kipf2016semi}, GNN-AK \cite{zhao2022from}, SUN \cite{frasca2022understanding}. 
We report the obtained MAE in Table \ref{table:graphCounting}, where we see that our RFP approach improves on the natural baselines, such as RNF and PE. Also, note that here, PowerEmbed \citep{huang2022from}, which performs several propagation and normalization steps, cannot be directly compared with RFP because of their reliance on input node feature rather than random node features, which are completely absent in this dataset. 

\subsection{Node classification on a large graph}
\label{app:additional_expr}
We evaluate RFP on the ogbg-arxiv dataset \cite{hu2020ogb}, which contains 169,343 nodes and 1,166,243 edges, and report the result in \cref{table:arxiv}. It was straightforward to train RFP on that dataset using a GCN backbone with P=16 propagation steps and k=64 eigenvectors. A feed-forward pass takes 61.44 milliseconds on an RTX-3090 Nvidia GPU. We also note, that while using our RFP requires more computations, when not learning the propagation operator (as in the experiment reported in \cref{table:arxiv}), our RFP does not incur significant memory overhead because it does not require performing backpropagation through the RFP layers, besides the initial embedding layer between the actual inputs node embedding and the RFP embedded trajectory, as shown in Equation (4) in our paper. More precisely, a 2-layer GCN requires 7635 MB while our RFP combined with a 2-layer GCN backbone requires 9268 MB of memory.

\begin{table*}
    \caption{The obtained node classification accuracy(\%) on ogbn-arxiv. The exact computation of the Laplacian $\hat{\bfL}$ eigenvectors required several hours and therefore the result is not available.}
\label{table:arxiv}\vskip 0.1in
\center{
\footnotesize
  \begin{tabular}{cccc}
    \toprule
    GCN & RNF + GCN & GCN + $\hat{\bfL}$ EIGVECS \textsuperscript{$\dagger$} & GCN + RFP (Ours) \\
    \midrule
     71.08 & 70.91 & N/A & 71.97 \\  
    \bottomrule
  \end{tabular}}
  \end{table*}

From \cref{table:arxiv}, we can see that:

(1) Adding plain random node features (RNF) to GCN on the large ogbg-arxiv dataset leads to slightly inferior results compared to the baseline of GCN.

(2) Computing the Laplacian eigenvectors as positional encoding (PE) is not practical for large graphs, because the positional encoding computation required several hours to complete on ogbg-arxiv.

(3) Combining GCN with our random feature propagation (RFP) mechanism improves the accuracy by 0.89\%. We therefore deem that our RFP is also beneficial for large graphs for two reasons. First, it can be applied to large graphs, whereas using the exact Laplacian eigenvectors may not be practical, due to the computational overhead, and second, we found our RFP to empirically improve the baseline results.  

\subsection{Ablation study}
\label{app:ablations}

\paragraph{Influence of the number of eigenvectors $k$.} An important hyperparameter of our RFP is the number of eigenvectors to estimate, $k$. \cref{table:zinc12k_molhiv_ablation} shows the results on ZINC-12k and OGBG-MOLHIV datasets for different values of $k$. We further show the impact of $k$ on the synthetic Graph Counting dataset in \cref{table:graphCounting_ablation}, and on the node classification datasets in \cref{tab:ncAblation}. As can be seen in the tables, up to a typical threshold value of 16 and 64 on graph-level and node-level tasks, respectively, there is a significant benefit in the estimation of more eigenvectors. In the node classification datasets, which contain large graphs, we also consider the case where all eigenvectors are estimated, where we do not see considerably better or worse performances.

\paragraph{Influence of the number of propagations $P$.}
We study the impact of the hyperparameter $P$, i.e., the number of propagation we take in our RFP. In  \cref{tab:ncAblation} we show the node classification accuracy on several datasets for different values of $P$. We see that as a general rule of thumb, more propagations increase accuracy. This is in line with our interpretation of the RFP as the trajectory of generalized power methods. More specifically, more propagations are equivalent to more power method iterations, leading to more accurate eigenvector estimations of the considered graph operator $\hat{\bfA}$, $\hat{\bfL}$, or $\bfS^{\text{learn}}$. Also, we compare our RFP that propagates random features with a natural baseline of propagating the input node features, dubbed as NFP (node feature propagation). We see that our RFP offers superior results, further highlighting the importance of random node features.

\paragraph{Additional studies.} We propose two additional settings to our RFP. (1)
In addition to utilizing GINE as a backbone in our graph-level experiments, we also consider a combined transformer and GINE backbone, based on the GraphGPS architecture proposed by \citet{rampasek2022GPS}. To understand the impact of the architectural change only, unlike in \citet{rampasek2022GPS}, we do not use \emph{structural} encodings, but rather only our RFP positional encodings. (2) We consider the combination of our RFP with a domain-aware structural encoding method, namely, GSN \cite{bouritsas2022improving}. We report the obtained results on ZINC-12K and OGBG-MOLHIV, \cref{table:zinc12k_molhiv_ablation}, where we can see the following: (1) compared to using GINE as a backbone only, we obtain slightly better results with GraphGPS. For example, using GraphGPS with RFP - $\ell_2$ - $\hat{\bfL}, \hat{\bfA}$ reduces the test MAE on ZINC-12K from 0.1368 to 0.1322. (2) The improved results offered by combining RFP with GSN indicate that there is a synergy between the two methods. That is, the combination of both methods yields improved results that are closer to current state-of-the-art methods.

\section{Learnable Propagation Operator}
\label{app:learnableOperator}
To implement our learnable propagation operator, we utilize a multi-head self-attention mechanism \cite{vaswani2017attention}, denoted by $\rm{MHA}$. The $\rm{MHA}$ computes $h$ scores for each pair of nodes $(i,j) \in \mathcal{V}\times \mathcal{V}$. As an input to the $\rm{MHA}$, we use a concatenation of the input node features $\bff^{\text{in}}$ and the output of an MPNN applied to these input features, namely 
$\bff = \bff^{\text{in}} \oplus \rm{MPNN}(\bff^{\text{in}})$.
 In practice, we found that using a single GCN \cite{kipf2016semi} layer as the $\rm{MPNN}$ component works well.  We then obtain the pairwise attention scores by:
\begin{equation}
    \label{eq:mha}
    \rm{MHA}(\bff_i, \bff_j) = [\alpha_{ij}^{1}, \ldots, \alpha_{ij}^{H}]
\end{equation}
where $\bff_i$ denotes the $i$-th row of $\bff$. We then define the propagation operator as the average of these scores, i.e.\
$\bfS^{\text{learn}}_{ij} = \frac{1}{H} \sum_{h=1}^{H} \alpha_{ij}^{h}
$.
We then use the learnable operator $\bfS^{\text{learn}}$ in order to generate our trajectory. This is illustrated in Figure \ref{fig:learn_op}.

In our experiments, we found that the learnable operator may learn to encode different sparsity patterns (i.e., node interaction patterns), ranging from learning to choose a global node (that is connected to all of the nodes in the graph), to reducing neighborhood connectivity and enhancing self-loop weights. The former can be interpreted as learning to mimic the virtual global node approach, which is a well-known augmentation technique in GNNs \cite{bouritsas2022improving}.
Additionally, for fixed propagation operators, we have analyzed the information encoded in the trajectory in Section \ref{sec:theory}, and showed, that among other things, the trajectory to the eigenvectors provides us access to traces of powers of the propagation operators, which can be used to count relevant substructures such as cycles.

In the case of the learnable propagation operator, the trajectory provides at least two types of information:

(1) In a similar vein to the discussion in \cref{sec:theory}, the trajectory provides us information about the traces of powers of the learnable propagation operator. These traces can be interpreted as structural information that can be extracted from the new connectivity. It is possible, for example, to interpret the trace of the k-th power of the learnable operator as the probability that a random walk with k steps returns to the starting point. This provides structural information about a different, but still natural connectivity defined on the input graph compared to the original connectivity.

(2) Another possibly useful information that can be obtained from the trajectory of the learnable operator is the computation of features that describe the newly learnt connectivity and may provide useful hints for the downstream GNN. As shown \cref{sec:experiments}, the learnable operator consistently improves or offers on-par performance compared with using RFP with pre-defined operators, mostly obtaining the former, that is, improving performance.

\section{Theoretical Results}\label{app:theory}

\subsection{Triangle counting}\label{app:triangles}

It is well known that the number of triangles in a graph $G$ with adjacency matrix $\bfA$ is computed as $c_3 = \frac{\mathrm{Trace}(\bfA^3)}{6}$~\citep{harary1971on}, that is, it is a sixth of the trace of the cubed adjacency matrix. As the trace of a matrix corresponds to the sum of its eigenvalues, the number of triangles can be estimated directly from an estimate of the cubed eigenvalues of the adjacency matrix $\bfA$.
Our propagation scheme can, in principle, approximate these quantities by choosing operator $\bfS = \bfA^3$ and by letting the algorithm reach convergence. However, \citet{avron2010counting} argues that a significantly more efficient estimation of triangle counts in undirected graphs is constructed via Monte Carlo simulation, that is by running the following averaging operation:
\begin{align}
    \eta(G) &= \frac{1}{M} \sum_{i=1}^{M} T_i \\
    T_i &= ({\bfy}_i^{\top} \bfA {\bfy}_i) / 6 \\
    {\bfy}_i &= \bfA {\bfr}_i
\end{align}
\noindent where ${\bfr}_i$ is a $n \times 1$ vector with i.i.d.\ entries sampled from distribution $\mathcal{D}$. The author studies the impact of $\mathcal{D}$ on the error bound of the randomized algorithm, with particular emphasis on the cases where $\mathcal{D}$ is either a standard normal $\mathcal{N}(0,1)$ or a discrete Rademacher. 

For the latter choice, as we report in~\Cref{prop:implementing-trace-triangle}, we show that our framework can exactly implement the above randomized algorithm and, therefore, inherit the approximation error bounds derived by~\citet{avron2010counting}.

\begin{proof}[Proof of \Cref{prop:implementing-trace-triangle}]
    As for $\mathcal{H}$, let $\bfS = \bfA, P \geq 2, w \geq 3, B = M, k = 1$ ($P, w, B, k \in \mathbb{N}$) and $\bfr \in \{-1, +1\}^{n \times 1}$ have its entries sampled independently from a Rademacher distribution. Let ${\bfr}_i$ refer to the initial starting point for trajectory ${\bft}_b, b=1, \dots, B$. In output from our propagation procedure we have the set of trajectories $\{ {\bft}_1, \dots, {\bft}_b, \dots, {\bft}_B \}$, where, for any $b$, ${\bft}_b = {\bfa}^{(0)}_b \oplus {\bfa}^{(1)}_b \oplus \dots \oplus {\bfa}^{(P)}_b$. These are gathered, along with initial node features, into the matrix $ {\bfx}^{(0)} = ( \bigoplus_{b=1}^B {\bft}_i ) \oplus {\bfx} $, obtained by stacking together the trajectories from the propagation phase, as well as the initial node features $\bfx$.
    
    Importantly we note that, for any $b$, $\bfu_b = {\bfa}^{(0)}_b \oplus {\bfa}^{(1)}_b \oplus {\bfa}^{(2)}_b = {\bfr}_b \oplus \bfA {\bfr}_b \oplus \bfA^2 {\bfr}_b$ in view of our specific choice of propagation operator and hyperparameter $w$. The concatenation of vectors $\bfu_b$'s, i.e.\ $\bigoplus_{b=1}^B \bfu_b$ will constitute the input of the downstream GNN architecture; the initial embedding layer $\bff^{(0)}$ acts so extract these sub-trajectories from ${\bfx}^{(0)}$ (by discard all the other information originally contained therein). To this aim, it is sufficient to consider a linear embedding layer, that is, $\bff^{(0)} \equiv \bfK_\text{embed}$. This last weight matrix is constructed as $\bfK_\text{embed} = K \oplus {\bf0}_{2B,d_{in}}$, with $K$ a block diagonal matrix of dimension $2B \times BP$. Each of the $B$ blocks in $K$ is a $2 \times P$ matrix $J$ whose only non-zero entries are those at first row and second column and second row and third column (which are one-valued). Effectively, each submatrix $J$ has the role of `selecting' only the values corresponding to the first and second propagation steps.
    
    We now move to describe the downstream architecture. We start from the case where it is implemented as a (maximally expressive) message-passing architecture in the form proposed by \citet{morris2019weisfeiler}, that is as a stacking of $T$ layers updating node representations as:
    \begin{equation}\label{eq:graph_conv}
        {\bfx}^{(t+1)}_v = \sigma \big ( W^{(t)}_1 {\bfx}^{(t)}_v + \sum_{w \sim v} W^{(t)}_2 {\bfx}^{(t)}_w + b^{(t)}\big )
    \end{equation}
    \noindent where $\sigma$ is a ReLU non-linearity. The stacking is then followed by sum-readout aggregation operator and a final Multi Layer Perceptron (MLP) $\psi$ computing a graph level output:
    \begin{equation}\label{eq:graph_conv_readout}
        y = \psi \big ( \sum_v {\bfx}^{(T)}_v \big )
    \end{equation}
    
    The yet-to-describe architecture is able to start from $\bigoplus_{b=1}^B \bfu_b$, i.e. the output of ${\bff}^{(0)}$, and complete the computation of the value of interest, i.e. $c_3(G) = \frac{1}{6M} \sum_{b=1}^{M} ({\bfy}_b^{\top} \bfA {\bfy}_b)$, with ${\bfy}_b = \bfA {\bfr}_b$. We rewrite term ${\bfy}_b^{\top} \bfA {\bfy}_b$ as: ${\bfy}_b^{\top} \cdot (\bfA {\bfy}_b) = (\bfA {\bfr}_b) \cdot (\bfA \bfA {\bfr}_b) = (\bfA {\bfr}_b) \cdot ({\bfA}^2 {\bfr}_b)$, that is, the dot product between, respectively, the second and third term of the $b$-th trajectory. Effectively, this product evaluates to $\sum_v ({\bfz}_b)_v$, ${\bfz}_b = (\bfA {\bfr}_b) \odot ({\bfA}^2 {\bfr}_b)$, so that the expression is reformulated as: $c_3(G) = \frac{1}{6M} \sum_{b=1}^{M} ({\bfy}_b^{\top} \bfA {\bfy}_b) = \frac{1}{6M} \sum_{b=1}^{M} \sum_v ({\bfz}_b)_v$, ${\bfz}_b = (\bfA {\bfr}_b) \odot ({\bfA}^2 {\bfr}_b)$.
    
    We first show that terms ${\bfz}_b$'s can be computed by a stacking of layers of the form in \Cref{eq:graph_conv}; then, that the remaining computation can be performed by the readout module described in \Cref{eq:graph_conv_readout}. 
    
    Although element-wise products are not natively supported by feed-forward neural networks, it would be possible to approximate them in view of the universal approximation theorem~\citep{cybenko1989approximation,hornik1989multilayer}. However, in our case, we require our model to only \emph{memorize} the products between a finite number of operands, due to the discrete, finite nature of the Rademacher distribution that generates initial random features. In other words, we can exactly implement the calculation of ${\bfz}_b$'s thanks to the memorization theorem due to~\citet{yun2019small}:
    \begin{theorem}[Theorem 3.1 from~\citet{yun2019small}]\label{thm:memorisation_yun}
        Consider any dataset $\{ (x_i, y_i) \}_{i=1}^N$ such that all $x_i$'s are distinct and all $y_i \in [-1, +1]^{d_y}$. Let $f_{\boldsymbol{\theta}}$ be a $3$-layer ReLU-like MLP model, where $d_1, d_2$ refer to the dimensionality of its hidden representations. If $f_{\boldsymbol{\theta}}$ satisfies $4 \lfloor \nicefrac{d_1}{4} \rfloor \lfloor \nicefrac{d_2}{(4d_y)} \rfloor \geq N$, then there exist parameters $\boldsymbol{\theta}$ such that $y_i = f_{\boldsymbol{\theta}}(x_i)$ for all $i \in [N]$.
    \end{theorem}
    The above theorem states the existence of a (small) ReLU-like network that can perfectly fit a given finite, well-formed dataset. ReLU-like networks are MLPs whose non-linearities are in the form $\sigma_R(x) = s_+ x$ for $x \geq 0$, $s_- x$ for $x < 0$. Family $\sigma_R(\cdot)$ includes the standard ReLU activation. \Cref{thm:memorisation_yun} can also be extended to the case where targets $y_i$'s have their entries not necessarily constrained in the range $[-1, +1]$:
    \begin{proposition}[Proposition 14 from~\citet{frasca2022understanding}]\label{prop:memorisation_ours}
        Consider any dataset $D = \{ (x_i, y_i) \}_{i=1}^N$ such that all $x_i$'s $\in \mathbb{R}^{d_x}$ are distinct and all $y_i \in \mathbb{R}^{d_y}$. There exists a $3$-layer ReLU-like MLP $f^{(D)}$ such that $y_i = f^{(D)}(x_i)$ for all $i \in [N]$.
    \end{proposition}
    The memorization theorems are, in fact, applicable in our context: this is fundamentally due to the fact that random inits $\bfr$'s come from a probability distribution with finite support, and the set of all possible adjacency matrices associated with $n$-nodes graphs is finite, so that there only exists a finite number of possible ($2$-step) propagation trajectories. In specific, consider the set of $2$-dimensional vectors $\mathcal{X} = \Big \{ [ (\bfA \bfr)_i, (\bfA^2 \bfr)_i ] \enspace | \enspace i \in \{ 1, \dots, n \}, \bfA \in \{0, 1\}^{n \times n}, \bfr \in \{-1, +1\}^n \Big \}$. This set is finite as its elements are vectors whose entries can only assume a finite set of values. This is because $[ (\bfA \bfr)_i, (\bfA^2 \bfr)_i ] = [ \bfA_{i,:} \bfr, \bfA^2_{i,:} \bfr ]$, where $\bfr$ is from the finite $\{-1, +1\}^n$ and $\bfA_{i,:}, \bfA^2_{i,:}$ are rows from operators $\bfA, \bfA^2$ which are in a finite number due to $\bfA$ being in the finite $\{0, 1\}^{n \times n}$. Now, let dataset $D$ be constructed as $D = \big \{ (x, y) \thinspace | \thinspace x = [x_a, x_b] \in \mathcal{X}, y = x_a \cdot x_b \big \}$. Dataset $D$ `lists' all possible pairwise multiplications that need to be calculated by our network in the computation of terms $(\bfA {\bfr}) \odot ({\bfA}^2 {\bfr})$ above. $D$ satisfies the hypotheses of \Cref{prop:memorisation_ours}, which guarantees the existence of $3$-layer ReLU-like MLP $f^{(D)}$ memorizing it.

    Because of this, there also exists a $3$-layer ReLU MLP $f$ which, applied on the vectorized set of sub-trajectories $\bigoplus_{b=1}^B \bfu_b$, outputs the pairwise multiplications of the two channels in each of the $\bfu_b$'s, i.e. $\big ( (\bfA {\bfx}_1) \odot (\bfA^2 {\bfx}_1) \big ) \oplus \dots \oplus \big ( (\bfA {\bfx}_B) \odot (\bfA^2 {\bfx}_B) \big )$. Let $\big ( W^{(i)} \big )_{i=1}^{3}$ be $f$'s weight parameters; each weight matrix $W^{(i)}$ has a block diagonal structure with $B$ blocks, where each block exactly corresponds to matrix $W^{(i,D)}$ in $f^{(D)}$. As for bias terms they are obtained concatenating, $B$ times, the bias terms of $f^{(D)}$. The MLP $f$ is trivially implemented by $3$ message-passing layers in the form of \Cref{eq:graph_conv}: it suffices to let $W^{(t)}_1 = W^{(t)}, W^{(t)}_2 = \bf0$, for $t = 1, 2, 3$ and the biases $b^{(t)}$ coincide the ones of $f$.
    
    In output from these message-passing steps is $\bigoplus_{b=1}^B \bfz_b$. Here, we neglect the presence of the ReLU non-linearity in the last of these layers. This is legit assumption, as the effect of such activation can easily be nullified by properly choosing the weights in the preceding and subsequent linear transformations\footnote{Function $f(x) = R \cdot \sigma ( E \cdot x )$, defined on $\mathbb{R}^{d}$, computes the identity function for $E = [ \bfI_d | - \bfI_d ], R = \bfI_d \oplus \bfI_d$, $\sigma$ being the ReLU non-linearity and $[ \cdot | \cdot ]$ representing vertical concatenation.}.
    
    The computation of the dot product, for each sample $b$, is completed by the sum-readout operation in \Cref{eq:graph_conv_readout}, which effectively outputs vector $\bigoplus_{b=1}^B (6T_b)$. Lastly, $\psi$ can easily and exactly implement the averaging operation $\frac{1}{6M} \sum_{b=1}^{M} 6T_b$, completing the computation of $c_3$: it only requires one dense layer with $1 \times B$ weight matrix $W_\psi = \frac{1}{6B} \vec{1}_B^{\top}$.

    In the case of a DSS-GNN architecture~\citep{bevilacqua2021equivariant}, a bag of $B$ subgraphs is formed. Subgraph $b$ retains the same connectivity and initial node features of the input graph, but the latter ones are augmented with the trajectory resulting from the $b$-th random initialization. In other words, $A^b = A$, while $\bfx^{b,(0)} = \bft_b \oplus \bfx$. Each $\bfx^{b,(0)}$ is embedded by the same ${\bff}^{(0)}$ layer, and then processed by DSS-GNN layers, which, equipped with maximally expressive graph convolutional layers in the form of \Cref{eq:graph_conv}, read as:
    \begin{equation}\label{eq:dss-gnn}
        \bfx^{k,(t+1)}_v = \sigma \big ( W_{1,t}^{1} \bfx^{k,(t)}_v + W_{2,t}^{1} \sum_{w \sim v} \bfx^{k,(t)}_w + W_{1,t}^{2} \sum_{b=1}^B \bfx^{b,(t)}_v +  W_{2,t}^{2} \sum_{w \sim v} \sum_{b=1}^B \bfx^{b,(t)}_w \big)
    \end{equation}
    In DSS-GNN, a stacking of layers in the form of \Cref{eq:dss-gnn}, is followed by a subgraph readout operation:
    \begin{equation}\label{eq:dss-gnn-subgraph-readout}
        \bfx^{k} = \rho \big (\sum_v \bfx^{k,(T)}_v \big )
    \end{equation}
    \noindent and a final, global, readout layer:
    \begin{equation}\label{eq:dss-gnn-readout}
        y = \psi \big (\sum_{b=1}^{B} \bfx^{b} \big )
    \end{equation}
    
    This architecture can implement the triangle counting algorithm. First, we consider the same ${\bff}^{(0)}$ embedding layer we illustrated for the GNN architecture above. Then, a stacking of $3$ DSS-GNN layers computes the element-wise product $(\bfA {\bfr}_b) \odot ({\bfA}^2 {\bfr}_b)$, for each $b=1, \dots, B$. Indeed, it is sufficient to let $W^1_{1,t} = W^{(t,D)}$, all remaining weight matrices $W^1_{2,t}, W^2_{1,t}, W^2_{2,t} = \bf0$, with $t=1, 2, 3$ and all biases match those of $f^{(D)}$. Similar considerations w.r.t.\ the last ReLU activation as the ones above apply here.

    The sum aggregation in \Cref{eq:dss-gnn-subgraph-readout} completes the computation of the dot-product producing a scalar value for each subgraph. When $\rho$ is a single linear layer with $1 \times 1 $ weight matrix $W_\rho = [\frac{1}{6}]$, the subgraph readout operation on subgraph $b$ outputs value $T_b$.

    Finally, the readout operation in \Cref{eq:dss-gnn-readout} computes the averaging $\frac{1}{B} \sum_{b=1}^{B} T_b$, where, again $\psi$ is a single linear layer with a $1 \times 1$ weight matrix $W_\psi = [\frac{1}{B}]$.    
\end{proof}

An interesting aspect of the above proof is that the discussed constructions do not leverage any message-passing operation. Strictly speaking, the only architectural requirements to implement the desired dot products are: (i) a shared MLP applied to all node features in parallel, and (ii) a readout operation followed by a dense layer. Because of this, when starting from the same inputs, even a DeepSets model~\citep{zaheer2017deep} would, in principle, be able to implement the $\textsc{TraceTriangle}_R$ algorithm. This observation further confirms how meaningful structural patterns can easily be extracted from early propagation steps; in practice, downstream message-passing operations can then further refine this information into higher-level features.

As we will see next, the direct consequence of \Cref{prop:implementing-trace-triangle} is that, as for the $\textsc{TraceTriangle}_R$ algorithm, our method can $(\epsilon, \delta)$-approximate the triangle counting function. This result is directly adapted from \citep{avron2010counting}, and relies on the ability to implement the Hutchinson trace estimator:
\begin{definition}\label{def:hut}
    An \emph{Hutchinson trace estimator} for a symmetric matrix $\bfC \in \mathbb{R}^{n \times n}$ is:
    \begin{equation}
        H_M ( \bfC ) = \frac{1}{M} \sum_{i=1}^{M} {\bfz}^{\top}_{i} \bfC {\bfz}_{i}
    \end{equation}
    \noindent where ${\bfz}_i$'s are independent vectors in $\mathbb{R}^{n}$ whose entries are i.i.d.\ Rademacher random variables.
\end{definition}

The Hutchinson estimator has been analyzed in~\citep{avron2011randomized,avron2010counting}, where the authors prove the following
\begin{lemma}[Lemma 5 in \citep{avron2010counting}]\label{thm:hut-eps-delta-approx}
    Let $\bfC$ a symmetric matrix with non-zero trace. For $M \geq 6 \epsilon^{-2} \rho(\bfC)^2 \ln ( 2 \text{rank}(\bfC) / \delta )$, the Hutchinson estimator $H_M(\bfC)$ is an $(\epsilon, \delta)$ approximator of $\text{trace}(\bfC)$, that is:
    \begin{equation}\label{eq:eps-delta-trace}
        \text{Pr} \big (\ |H_M(\bfC) - \text{trace}(\bfC)| \leq \epsilon \cdot \text{trace}(\bfC)\ \big ) \geq (1 - \delta)
    \end{equation}
    \noindent where, for a symmetric matrix $\bfC \in \mathbb{R}^{n \times n}$, $\rho(\bfC) = \frac{\sum_{i=1}^{n} |\lambda_i|}{\text{trace}(\bfC)}$, and $\lambda_i$ refers to the $i$-th eigenvalue of $\bfC$.
\end{lemma}

The following proposition relates the ability of our RFP algorithm to implement $\textsc{TraceTriangle}_R$ and that of computing a specific Hutchinson estimator:
\begin{proposition}\label{prop:hut-est}
    Let $\bfA$ be the adjacency matrix associated with graph $G \in \mathcal{G}_n$. There exists a choice of hyperparameters $\mathcal{H}$, and weights $\mathcal{W}$, such that our RFP algorithm (ref. \ref{alg:rfp}), equipped with either a DSS-GNN~\citep{bevilacqua2021equivariant} or GraphConv~\citep{morris2019weisfeiler} downstream architecture, computes the quantity $\frac{1}{6} H_b ( {\bfA}^3 )$, i.e., one sixth of the Hutchinson's estimator for $\text{trace}( {\bfA}^3 )$.
\end{proposition}
\begin{proof}
    From \Cref{prop:implementing-trace-triangle} our approach can implement $\textsc{TraceTriangle}_R$, that is, it outputs, for graph $G$, the value $\eta(G)$. At the same time, $\eta(G)$ can be rewritten as: $\eta(G) = \frac{1}{B} \sum_{b=1}^{B} T_b = \frac{1}{B} \sum_{b=1}^{B} \frac{{\bfy}^{\top}_b \bfA {\bfy}_b}{6} = \frac{1}{6} \frac{1}{B} \sum_{b=1}^{B} {\bfy}^{\top}_b \bfA {\bfy}_b = \frac{1}{6} \frac{1}{B} \sum_{b=1}^{B} ({\bfx}^{\top}_b {\bfA}^T) \bfA (\bfA {\bfx}_b)$. As $\bfA$ is symmetric, we have: $\eta(G) = \frac{1}{6} \frac{1}{B} \sum_{b=1}^{B} {\bfx}^{\top}_b (\bfA \bfA \bfA) {\bfx}_b = \frac{1}{6} H_B ({\bfA}^3)$.
\end{proof}

From the above proposition we immediately derive the following
\begin{corollary}\label{cor:approximating-triangle-counts}
    Let $\delta > 0$ a failure probability and $\epsilon > 0$ a relative error. There exists a choice of hyperparameters $\mathcal{H}$ with $B \geq 6 \epsilon^{-2} \rho(\bfA^3)^2 \ln ( 2 \text{rank}(\bfA^3) / \delta )$ and a choice of weights $\mathcal{W}$ such that, for a graph $G \in \mathcal{G}_n$ associated with adjacency matrix $\bfA$, the output $y_G$ of our RFP algorithm (ref. \ref{alg:rfp}) is a $(\epsilon, \delta)$ approximation of $c_3(G)$, i.e. the number of triangles of $G$:
    \begin{equation}\label{eq:eps-delta-triangles}
        \text{Pr} \big (\ |y_G - c_3(G)| \leq \epsilon \cdot c_3(G)\ \big ) \geq (1 - \delta).
    \end{equation}
\end{corollary}
\begin{proof}[Proof of \Cref{cor:approximating-triangle-counts}]
    $\bfA^3$ is a symmetric matrix, since $G$ is undirected and $\bfA$ is, thus, a symmetric matrix. From \Cref{prop:hut-est}, with $\bfC = \bfA^3$, with have:
    \begin{equation}
        \text{Pr} \big (\ |H_B(\bfA^3) - \text{trace}(\bfA^3)| \leq \epsilon \cdot \text{trace}(\bfA^3)\ \big ) \geq (1 - \delta)
    \end{equation}
    \noindent with $B \geq 6 \epsilon^{-2} \rho(\bfA^3)^2 \ln ( 2 \text{rank}(\bfA^3) / \delta )$. The event associated with expression (i) $|H_B(\bfA^3) - \text{trace}(\bfA^3)| \leq \epsilon \cdot \text{trace}(\bfA^3)$ is the same associated with expression (ii) $|y_G - c_3(G)| \leq \epsilon \cdot c_3(G)$, this proving the corollary. This is because we can rewrite (i) as $|6 \cdot y_G - \text{trace}(\bfA^3)| \leq \epsilon \cdot \text{trace}(\bfA^3)$, from which we obtain (ii) by multiplying both sides by the positive constant $\nicefrac{1}{6}$: $|y_G - \frac{1}{6}\text{trace}(\bfA^3)| \leq \epsilon \cdot \frac{1}{6} \text{trace}(\bfA^3)$, and by recalling that $c_3(G) = \frac{\text{trace}(\bfA^3)}{6}$.
\end{proof}

\subsection{Universal approximation}\label{app:universal_approximation}

\citet{abboud2020surprising} and \citet{puny2020global} study the expressive power of the RNI model proposed by \citet{sato2021random}. The model essentially applies a message-passing neural network on an input graph where node attributes are augmented with random features, acting as `identifiers':
\begin{equation}\label{eq:rni}
    y_G = g \big ( \bfA, \bff^{\text{in}} \oplus \bfr \big )
\end{equation}
\noindent where $g$ is an MPNN and $\bff^{\text{in}} \oplus \bfr$ is the horizontal concatenation of the initial node attribute matrix $\bfx$ and the random feature matrix $\bfr$. Despite slightly different technical assumptions and details, both the two works show that this architecture is an $(\epsilon, \delta)$ universal approximator of functions on graphs. In this paper we take as reference the work by \citet{puny2020global}, and build upon their result to derive \Cref{prop:universal-approximation}:
\begin{proposition}\label{prop:universal-approximation}
    Let $\Omega_n$ be a compact set of graphs on $n$ vertices and $f$ be a continuous target function defined thereon. Then, if random node features are sampled from a continuous, bounded distribution with zero mean and finite variance $c$, for all $\epsilon, \delta > 0$, then there exists a choice of hyperparameters $\mathcal{H}$, and weights $\mathcal{W}$, such that:
    \begin{equation}
        \forall G \in \Omega_n: \text{Pr} \big ( |f(G) - y_G| \leq \epsilon \big ) \geq 1 - \delta
    \end{equation}
    where $y_G$ is the output of our algorithm on $G$ for the considered choice of $\mathcal{H}, \mathcal{W}$.
\end{proposition}
\begin{proof}[Proof of \Cref{prop:universal-approximation}]
    The proposition immediately follows by showing a choice of hyperparameters and weights which makes our algorithm satisfy the premises of \citet[Proposition 1]{puny2020global}, which can then be invoked to guarantee the $(\epsilon, \delta)$ universal approximation properties of our approach. As per our hypothesis, random features $\bfr$ are sampled from a continuous, bounded distribution; thus it is only left to show how algorithm can default to the RNI architecure in \Cref{eq:rni}. For any $P \geq 0$, it is sufficient to choose $B=1$ to let the algorithm operate in a single trajectory regime, and construct module $\bff^{(0)}$ in \Cref{eq:concatEmbedInputs} such that $\bff^{(0)} = \sigma({\bfK}_{\text{embed}}(\bff^{\text{in}} \oplus \bft)) = \bff^{\text{in}} \oplus \bfr$ (recall that the output of $\bff^{(0)}$ is subsequently processed by a downstream MPNN). This construction of $\bff^{(0)}$ is possible by properly choosing weights ${\bfK}_{\text{embed}}$ and by setting non-linearity $\sigma$ to be the identity function.
\end{proof}

What would constitute a distribution $\mathcal{D}$ satisfying the hypotheses of \Cref{prop:universal-approximation}  (and \citep[Proposition 1]{puny2020global})? A valid choice is represented, for example, by the uniform distribution $\mathcal{U}_{[-1,1]}$, which is indeed zero-meaned, has variance $\nicefrac{1}{3}$, is continuous, and bounded. The Rademacher distribution discussed above is, on the other hand, discrete. The reason why this may be problematic, from a theoretical standpoint, is that it can assign two nodes the same random features with non-zero probability. This may prevent the readout layer of the MPNN from unambiguously reconstructing the graph topology at hand, a step required in showing universal approximation. We remark that, in any case, the probability of a collision of initial Rademacher random features quickly decreases with the dimensionality of $\bfr$ and that, for Rademacher variables, specific expressiveness results can be proved nonetheless (see \Cref{app:triangles}).

\subsection{Convergence to the operator eigenvectors}
\label{app:convergence}
\begin{lemma}\label{lem:rank}
Let $X_i\in\mathbb{R}^n$, $1 \le i \le k$ be continuous i.i.d.\ random variables with a continuous joint probability density function $f:\mathbb{R}^{n\times k}\rightarrow \mathbb{R}_+$. Denote their concatenation as $X=[X_1,\dots,X_k]\in\mathbb{R}^{n\times k}$. Let $\mathbf{Q}=[q_1,\dots,q_k]\in\mathbb{R}^{n\times k}$ be the matrix obtained by concatenating $k$ arbitrary orthonormal vectors. Then $\text{Pr}\left( \text{rank}(\mathbf{Q}\mathbf{Q}^\top X)=k \right)=1$.
\end{lemma}
\begin{proof}
We show that the complement event, namely $\{\text{rank}(\mathbf{Q}\mathbf{Q}^\top X)<k\}$ has zero probability. Let $\mathbf{Y}=\mathbf{Q}\mathbf{Q}^\top X$, then $\text{rank}(\mathbf{Y})<k$ if and only if $\text{rank}(\mathbf{Y}^\top\mathbf{Y})<k$, which happens if and only if $\text{det}(\mathbf{Y}^\top\mathbf{Y})=0$.
Therefore we need to show that the probability of the event $\{\text{det}(\mathbf{Y}^\top\mathbf{Y})=0 \}$ is zero.
Denote $g(\mathbf{Y})=\text{det}(\mathbf{Y}^\top\mathbf{Y})$, then $g$ is a polynomial and we are interested in the probability of its zero-level set. But the zero-level set of a non-zero polynomial has zero Lebesgue measure. In our case, $g$ is not the constant zero function since $g(\mathbf{Q})=1$. The result follows from the fact that  $\text{Pr}\left( g(\mathbf{Y})=0\right)=\int_{g^{-1}(\{0\})}f(x)dx=0$, where the integral is zero since we integrate over a zero measure set.
\end{proof}

\begin{proof}[Proof of \Cref{theo:convergence}]
We prove the proposition by showing we meet the assumptions of \citet[Theorem 5.1]{saad2011numerical}. Although the theorem shows convergence to the Schur vectors $[q_1, \dots, q_k]$ associated with $\lambda_1, \dots, \lambda_k$, 
 we note that for symmetric matrices they coincide with the $k$ orthogonal eigenvectors.

We only need to show that $\text{rank}(P_i [X_1, \dots, X_i]) = i, 1 \le i \le k$ where $P_i$ is the \emph{spectral projector} associated with eigenvalues $\lambda_1, \dots, \lambda_i$.
Since the eigenvalues are distinct by assumption, implying they are also \emph{simple}, and the operator is symmetric (left and right eigenvectors coincide), the spectral projector can be written in terms of the corresponding eigenvectors as as $P_i = \sum_{j =1}^i q_j q_j^\top$, where we assume w.l.o.g. that $q_j$ has unitary norm. Equivalently, we can write $P_i = \mathbf{Q}_i\mathbf{Q}_i^\top$ where $\mathbf{Q}_i = [q_1,\dots,q_i]$.
Therefore we need to show that $\text{rank}(\mathbf{Q}_i\mathbf{Q}_i^\top [X_1, \dots, X_i]) = i, 1 \le i \le k$ . This happens with probability 1 according to \Cref{lem:rank}. 
\end{proof}

\subsection{Computational complexity} 
\label{appx:complexity}

If $n, m$ refer to, respectively, the number of nodes and of edges of the input graph, the forward pass complexity of our method amounts to $T(n,m) = T_1(n,m) + T_2(n,m) + T_3(n,m)$, where $T_1$ is the complexity of sampling random node features, $T_2$ that of their propagation, and $T_3$ the complexity of the downstream GNN model. When $k \cdot B$ initial random features are sampled from a discrete distribution (such as a Rademacher one), then each node can be initialized in constant time via the Alias method~\citep{walker1974new}, so that $T_1(n,m) = \mathcal{O}(k B n)$. $T_2$ accounts for two relevant operations, i.e. a normalization and a propagation, so that $T_2(n,m) = \mathcal{O} \big ( P \cdot ( k^2 B n + k B n^2 ) \big )$, where $k^2 B n$  is the complexity of the QR normalization step~\citep{saad2011numerical} and $k B n^2$ that of the propagation one. This last can be reduced to $k B m$ in the case of a fixed operator from a sparse graph. Lastly, the complexity of the downstream GNN is, generally, $C(n,m) = \mathcal{O}(L \cdot d m)$ for $d$-dimensional hidden representations and $L$ convolutional layers; $\mathcal{O}(L \cdot dH n^2)$ in the case a Transformer architecture with $H$ heads. When working on a large-scale, sparse graph such that parameters $P, L, k, B, d$ can be considered small constants, the propagation and message-passing operations constitute the dominant factors in the computational complexity, which remains asymptotically linear in the number of edges.

\section{Algorithm}
\label{app:alg}
\begin{algorithm}[tb]
   \caption{A feed-forward pass with a single RFP trajectory and a GNN backbone.}
   \label{alg:rfp}
\begin{algorithmic}
   \STATE {\bfseries Input:} 
   Probability distribution $\mathcal{D}$,
   adjacency matrix $\bfA \in \mathbb{R}^{n \times n}$, input node features $\bff^{\text{in}} \in \mathbb{R}^{n \times c^{\text{in}}}$, propagation operator (predefined or learnable) $\bfS \in \mathbb{R}^{n\times n}$, normalization function $\bfN: \mathbb{R}^{n\times k} \rightarrow \mathbb{R}^{n\times k}$, frequency of the normalization $w \geq 1$, number of propagation steps $P \geq 1$, a learnable embedding layer $\bfK_{\text{embed}}$ and a GNN model  $\text{GNN}_{\mathcal{W}}$ with learnable parameters ${\mathcal{W}}$.
   \\\hrulefill
   \STATE Initialize $\bfr \in \mathbb{R}^{n \times k}$ as $\bfr \sim \mathcal{D}$.
   \STATE $\bfa^{(0)} = \bfr$
   \FOR{$p=1$ {\bfseries to} $P$}
   \STATE $\hat{\bfa}^{(p)} = \; \bfS \bfa^{(p-1)}$
   \IF{$\text{mod}(p, w) = 0$}
   \STATE ${\bfa}^{(p)} = \bfN(\hat{\bfa}^{(p)}) $
   \ELSE
   \STATE ${\bfa}^{(p)} = \hat{\bfa}^{(p)} $
   \ENDIF
   \ENDFOR
   \STATE $\bft = \bfr \oplus \bfa^{(1)} \oplus \ldots \oplus \bfa^{(P)}$
   \STATE $\bff^{(0)} = \sigma(\bfK_{\text{embed}}(\bff^{\text{in}} \oplus \bft))$
   \STATE {\textbf{return}} $\text{GNN}_{\mathcal{W}}(\bff^{(0)})$
\end{algorithmic}
\end{algorithm}

In this section we present the pseudocode of our RFP framework.
In particular, in \Cref{alg:rfp} we consider the case of a single trajectory (i.e.\ $B=1$), coupled with a GNN architecture. Notably, the pseudocode can be easily extended to handle multiple trajectories $B > 1$, by simply repeating the propagation-normalization steps $B$ times to obtain multiple trajectories $\{ \bft_1, \ldots, \bft_{B} \}$, $\bft_b \in \mathbb{R}^{n \times k(P+1)}$. This set of trajectories can then be either concatenated and fed to the GNN, or processed with a DSS-GNN architecture, as explained in \Cref{sec:archs}.

\end{document}